\documentclass{article}
\usepackage{iclr2027_conference,times}
\iclrfinalcopy

\usepackage{amsmath,amssymb,amsthm,mathtools,booktabs,microtype}
\usepackage{algorithm}
\usepackage{float}
\usepackage{algpseudocode}
\usepackage{enumitem}
\usepackage{graphicx}
\usepackage{hyperref}
\usepackage{url}
\usepackage{booktabs}
\usepackage{tabularx}
\usepackage{wrapfig}
\usepackage{capt-of}
\usepackage{placeins}
\hypersetup{hidelinks,
  pdftitle={Score-Preserving Continual Lasso under Task Shift},
  pdfauthor={Jingyi Guo, Runze Li, Lingzhou Xue}}

\newtheorem{theorem}{Theorem}
\newtheorem{proposition}{Proposition}
\newtheorem{example}{Example}
\newtheorem{lemma}{Lemma}
\newtheorem{corollary}{Corollary}

\newtheorem{condition}{Condition}

\newcommand{\R}{\mathbb{R}}
\newcommand{\norm}[1]{\left\lVert #1\right\rVert}
\newcommand{\argmin}{\mathop{\mathrm{arg\,min}}}
\newcommand{\E}{\mathbb{E}}
\newcommand{\Pbb}{\mathbb{P}}
\newcommand{\risk}{\mathcal{R}}
\newcommand{\mse}{\mathrm{MSE}}

\title{Score-Preserving Continual Lasso under Task Shift}

\author{Jingyi Guo \quad Runze Li \quad Lingzhou Xue\\
Pennsylvania State University\\
\texttt{jxg6178@psu.edu} \quad \texttt{rzli@psu.edu} \quad \texttt{lzxue@psu.edu}}

\begin{document}
\raggedbottom
\maketitle
\lhead{Preprint}

\begin{abstract}
Continual learning (CL) aims to learn from sequential tasks while mitigating catastrophic forgetting when historical data are no longer accessible. We study continual learning for high-dimensional sparse linear regression, where $\ell_1$ regularization creates a difficulty absent from continual ridge regression: that recursively reusing an \(\ell_1\)-penalized estimator can propagate earlier shrinkage and accumulate regularization bias across tasks. We propose score-preserving continual Lasso (SP-CL), which separates sparse estimation from state transmission by preserving the historical squared-loss score. Under squared loss, its cumulative estimator exactly recovers a full-history Lasso minimizer. When task coefficients coincide, it inherits the full-history sparse rate. For related task-specific coefficients, SP-CL uses a regularized current-task correction that leaves the transmitted state unchanged and attains the offline source-assisted rate under an oracle related-task condition. We also derive Overall/BWT/FWT decompositions that distinguish historical information loss, task conflict, and current-task specialization, and validate the resulting regimes in simulation.
\end{abstract}

\section{Introduction}
\label{sec:introduction}

Continual learning (CL) asks a model to learn from a sequence of tasks when
data from earlier tasks are no longer fully accessible. A central difficulty
is \emph{catastrophic forgetting}: adapting to a new task can degrade
performance on tasks learned earlier. At the same time, excessive stability
can prevent useful adaptation. A successful continual learner must therefore
maintain strong performance over observed tasks, preserve useful earlier
knowledge, and exploit that knowledge on new tasks. These goals are commonly
reflected by Overall performance, backward transfer (BWT), and forward
transfer (FWT), respectively \citep{parisi2019,wang2024survey}.

Regularization-based methods address this stability--adaptation problem by
constraining new solutions toward parameters learned from previous tasks.
For two fixed-design regression tasks, \citet{li2023fixed} derived tight risk
bounds for $\ell_2$-regularized updates and characterized the
forgetting--intransigence trade-off. \citet{zhao2024} extended this perspective
to sequences of tasks through generalized matrix-valued $\ell_2$
regularization, obtaining offline-oracle-order estimation while balancing
forward and backward transfer. In a complementary realizable setting,
\citet{levinstein2025} showed that fixed or increasing isotropic regularization
can achieve near-optimal or optimal task-horizon convergence rates.

More recent work has examined genuinely high-dimensional regimes.
\citet{zhao2025ridge} derived asymptotic prediction risks when dimension grows
proportionally with task sample size, characterizing Overall performance, BWT,
and FWT under a common coefficient and heterogeneous task covariances.
\citet{karpel2026} further allowed multiple task-specific linear teachers and
characterized optimal isotropic $\ell_2$ regularization in overparameterized
continual regression. Thus, existing theory now covers multi-task,
high-dimensional, and heterogeneous-task settings, but its continual
regularization mechanism remains quadratic rather than sparse. Table~\ref{tab:positioning} summarizes these
representative theoretical settings and capabilities most relevant to our
problem.

\begin{table}[!t]
    \centering
    \caption{
Comparison with representative theoretical work on regularization-based
continual linear regression.
``Task-specific coefficients'' allows different population coefficients across
tasks; ``Current-task correction'' denotes an explicit second-stage correction
using the current task.
}
    \label{tab:positioning}
    \scriptsize
    \setlength{\tabcolsep}{3.8pt}
    \renewcommand{\arraystretch}{1.08}
    \begin{tabular}{@{}lcccc@{}}
        \toprule
        \textbf{Work}
        & \textbf{Sparse}
        & \textbf{Task-specific}
        & \textbf{Current-task}
        & \textbf{Overall/BWT/FWT} \\
        &
        &
        \textbf{coefficients}
        & \textbf{correction}
        & \\
        \midrule

        \citet{li2023fixed}
        & $\times$
        & $\times$
        & $\times$
        & $\times$ \\

        \citet{zhao2024}
        & $\times$
        & $\times$
        & $\times$
        & $\checkmark$ \\

        \citet{zhao2025ridge}
        & $\times$
        & $\times$
        & $\times$
        & $\checkmark$ \\

        \citet{karpel2026}
        & $\times$
        & $\checkmark$
        & $\times$
        & $\times$ \\

        \midrule

        \textbf{SP-CL}
        & $\checkmark$
        & $\checkmark$
        & $\checkmark$
        & $\checkmark$ \\

        \bottomrule
    \end{tabular}
\end{table}

Empirical work uses sparse computation to reduce continual interference without
replay or task information \citep{bricken2023sdm}. We study its statistical
counterpart: high-dimensional predictors with sparse task-relevant coefficients,
where $\ell_1$ regularization creates a state-transmission problem absent from
quadratic methods.

High-dimensional streaming methods support $\ell_1$-regularized estimation
from sequential summaries but do not study continual task shift \citep{luo2021},
while sparse transfer learning permits coefficient adaptation but assumes that
source samples remain accessible \citep{li2022transfer}.

Three issues remain for sparse continual regression.

\emph{First, sparsity changes what should be transmitted.}
Under squared loss, historical data are exactly represented by a quadratic
objective through its cumulative curvature and score, so a quadratic state can
preserve the information needed from inaccessible data. With $\ell_1$
regularization, however, the fitted parameter also contains a penalty
subgradient; recursively transmitting that penalized center can therefore
propagate shrinkage that does not belong to the historical data objective.

\emph{Second, task variation creates a distinct adaptation problem.}
Existing high-dimensional $\ell_2$ theory can accommodate task-specific
coefficients by tuning regularization to balance historical retention and
current adaptation \citep{karpel2026}. Our question is different: can related
history first reduce estimation variance through cumulative information, and
can the resulting pooled bias then be corrected toward the current sparse
coefficient without revisiting historical samples?

\emph{Third, task heterogeneity complicates the interpretation of forgetting.}
Old-task degradation may reflect genuine loss of historical information, but it
may also arise because the predictor deliberately moves toward a different
current-task optimum. Overall performance, BWT, and FWT must therefore be
interpreted jointly when comparing cumulative retention with current-task
specialization.

In this paper, we aim to answer the following question:

\begin{center}
\emph{
Can sparse continual learning preserve full-history information without
accumulating past shrinkage, while adapting efficiently to related tasks?
}
\end{center}

We propose \emph{score-preserving continual Lasso} (SP-CL), which separates
sparse estimation from state transmission: $\ell_1$ regularization produces a
sparse cumulative estimator, while a corrected quadratic state preserves only
the historical squared-loss score---it \emph{regularizes locally but transmits
the data score globally}. Under task shift, SP-CL uses historical aggregation
to reduce estimation uncertainty and a sparse current-batch correction to adapt
toward the current coefficient without modifying the transmitted state. For evaluation, we derive exact identities that decompose Overall, FWT, and BWT into current-task gain, historical cost, task conflict, and estimation error.

Our contributions are threefold.
\begin{itemize}
  \item \textbf{Score-preserving sparse transmission.} We identify how
  recursively transmitting an $\ell_1$-penalized estimate propagates past
  penalty subgradients, and develop SP-CL to remove this contamination from the
  transmitted state. Its cumulative estimator is exactly a full-history
  pooled-Lasso minimizer, yielding the pooled-Lasso sparse rate under a shared
  coefficient.

  \item \textbf{Efficient adaptation under task shift.} Under task shift,
  SP-CL uses a current-batch sparse correction to adapt the cumulative estimate
  without modifying the transmitted state. Under a related-task condition, the
  resulting specialist attains the offline two-stage rate using
  only the continual state and current batch.

  \item \textbf{Continual evaluation under heterogeneity.} We separate
  catastrophic forgetting from task conflict in standard CL evaluation. Exact
  decompositions of Overall performance, BWT, and FWT distinguish historical
  information loss, task conflict, and current-task specialization, and quantify
  the current-gain versus historical-cost trade-off between cumulative and
  current-adapted prediction.
\end{itemize}

\section{Problem Setup and Evaluation}
\label{sec:stream}

\paragraph{Data and continual-learning protocol.}
We consider a standard continual-learning problem in which tasks
$t=1,\ldots,T$ arrive sequentially. At task $t$, the learner observes
$\mathcal D_t=\{(x_{ti},y_{ti})\}_{i=1}^{n_t}$ generated from
\begin{equation}
    y_{ti}=x_{ti}^{\top}\beta_t^\star+\varepsilon_{ti},
    \qquad x_{ti}\in\mathbb R^p,
    \label{eq:model}
\end{equation}
where $\beta_t^\star$ is the true parameter for task $t$ and may vary across
tasks. We focus on a high-dimensional sparse regime in which $p$ may exceed
the per-task sample size $n_t$, while the task-relevant coefficients are
sparse. Write
$X_t=(x_{t1},\ldots,x_{tn_t})^\top$,
$Y_t=(y_{t1},\ldots,y_{tn_t})^\top$, and
$N_t=\sum_{j=1}^t n_j$.
When task $t$ arrives, the learner may use the current batch $\mathcal D_t$
and its transmitted state, but cannot revisit raw observations from earlier
tasks. Distributional and task-relatedness conditions are introduced later
when needed for the statistical guarantees.

\paragraph{Task performance and continual-learning metrics.}
For regression, we measure task performance by the population mean-squared
prediction error. For a coefficient $b\in\R^p$ and a fresh task-$j$
observation $(x,y)$, write $\Sigma_j=\E_j(xx^{\mathsf T})$,
$\mse_j(b)=\E_j[(y-x^{\mathsf T}b)^2]$ and
$\risk_j(b)=\mse_j(b)-\mse_j(\beta_j^\star)
=\norm{b-\beta_j^\star}_{\Sigma_j}^2$. Thus, same-task differences in MSE
equal differences in excess prediction error. For tasks observed through time
$t$, let $\pi_{j,t}=n_j/N_t$, define
$\overline\mse_t(b)=\sum_{j=1}^t\pi_{j,t}\mse_j(b)$, and let
$w_t^\star\in\argmin_b\overline\mse_t(b)$ denote the best single population
coefficient for the tasks observed so far. With equal batch sizes,
$\overline\mse_t$ is the usual average task MSE.

For predictors $\{\hat b_t\}_{t=1}^T$ produced after successive tasks, we
consider the standard Overall, backward-transfer (BWT), and forward-transfer
(FWT) metrics. Here $\hat\beta_t^{\mathrm{cur}}$ denotes the current-only
Lasso benchmark defined in Section~\ref{sec:three-estimators}:
\par\noindent
\begin{minipage}[t]{0.63\linewidth}
  \raggedright
  $\displaystyle \mathrm{BWT}_t(\hat b)
  =\frac{1}{t-1}\sum_{j<t}
  \{\mse_j(\hat b_j)-\mse_j(\hat b_t)\}$
\end{minipage}\hfill
\begin{minipage}[t]{0.35\linewidth}
  \raggedleft
  $\displaystyle \mathrm{Overall}_t(\hat b)
  =\frac{1}{t}\sum_{j=1}^t\mse_j(\hat b_t)$
\end{minipage}
\par\vspace{2pt}\noindent
\begin{minipage}[t]{0.63\linewidth}
  \raggedright
  $\displaystyle \mathrm{FWT}_t(\hat b)
  =\frac{1}{t-1}\sum_{k=2}^t
  \{\mse_k(\hat\beta_k^{\mathrm{cur}})-\mse_k(\hat b_k)\}$
\end{minipage}
\par
Overall evaluates the current predictor on the tasks seen so far; for $t\ge2$,
BWT compares old-task performance with its task-time reference, and FWT compares
each predictor with current-only fitting. Positive BWT or FWT indicates
improvement relative to its reference; the three metrics need not rank predictors
in the same way when task parameters differ.

\section{Score-preserving continual Lasso}
\label{sec:method}

\subsection{From \texorpdfstring{$\ell_2$}{l2} Regularization to Sparse State Transmission}
\label{sec:quadratic-state}
\label{sec:direct-transmission}

Regularization-based continual regression commonly updates the current
parameter by penalizing its deviation from an estimate carried from earlier
tasks. Generalized $\ell_2$ regularization uses a matrix-weighted penalty of
the schematic form
\[
 \widehat w_t^{\mathrm{GR}}\in\argmin_w
 \left\{\ell_t(w)+\frac12\norm{w-\widehat w_{t-1}^{\mathrm{GR}}}_{A_t}^2\right\},
 \qquad \norm{v}_{A}^2=v^{\mathsf T}Av,\quad A\succeq0,
\]
where $\ell_t$ is the current-task loss \citep{zhao2024}. For squared loss,
this quadratic penalty provides a natural way to carry information from
earlier tasks. The following lemma identifies the curvature and score of the
full historical least-squares objective.

\begin{lemma}[Exact quadratic representation of historical squared loss]
\label{lem:quadratic-closure}
Let $H_t:=\sum_{j=1}^t X_j^{\mathsf T}X_j$ and
$G_t:=\sum_{j=1}^t X_j^{\mathsf T}Y_j$. For any $u_t\in\R^p$ satisfying
$H_tu_t=G_t$ (such a vector exists), and for every $w\in\R^p$,
\begin{equation}
  \frac12\sum_{j=1}^t\norm{Y_j-X_jw}_2^2
  =\frac12\norm{w-u_t}_{H_t}^2+C_t,
  \label{eq:quadratic-closure}
\end{equation}
where $C_t$ does not depend on $w$.
\end{lemma}

The proof is given in Appendix~\ref{app:method-proofs}. Hence $(H_t,u_t)$
represents the historical squared-loss objective up to a constant, and its
score is $S_t(w):=G_t-H_tw=H_t(u_t-w)$. We call $u_t$ a \emph{score carrier}.
This score contains only historical data information, not
first-order contributions from taskwise regularization. A penalized estimate,
however, need not satisfy this data-score equation.

For this subsection, write $\lambda_t:=\lambda_{w,t}$. The most direct sparse
analogue keeps the parameter-centered quadratic structure and adds an
$\ell_1$ penalty. We call it \emph{direct transmission of the penalized
center} (direct-center transmission):
\begin{equation}
  \widehat w_t^{\mathrm{dir}}
  \in\argmin_w\left\{
    \frac12\norm{Y_t-X_tw}_2^2
    +\frac12\norm{w-\widehat w_{t-1}^{\mathrm{dir}}}_{H_{t-1}}^2
    +N_t\lambda_t\norm{w}_1
  \right\},
  \label{eq:direct-stage}
\end{equation}
with $H_0=0$. Let
$z_t^{\mathrm{dir}}\in\partial\norm{\widehat w_t^{\mathrm{dir}}}_1$ be a KKT
subgradient.

\begin{proposition}[Exact accumulation of penalty subgradients]
\label{thm:direct-accumulation}
Set $G_0=H_0\widehat w_0^{\mathrm{dir}}=0$. Then, deterministically, for every
$t$,
\begin{equation}
  G_t-H_t\widehat w_t^{\mathrm{dir}}
  =\sum_{j=1}^t N_j\lambda_jz_j^{\mathrm{dir}}.
  \label{eq:direct-accumulation}
\end{equation}
\end{proposition}

The proof is given in Appendix~\ref{app:method-proofs}. The left-hand side is
the full-history data score, whereas the right-hand side consists entirely of
accumulated $\ell_1$-penalty subgradients. Thus, direct-center transmission
mixes persistent data information with task-local regularization. The terms
need not grow monotonically, but when they align the distortion can be severe,
as the following example shows.

\begin{example}[Support collapse under recursive penalization]
\label{ex:repeated-shrinkage}
Consider noiseless tasks with the same sparse coefficient
$\beta^\star\ne0$, equal batch size $n_0$, and a design that is orthogonal on
the active coordinates. In this setting, recursively transmitting the
penalized center eventually sets every coordinate in
$\operatorname{supp}(\beta^\star)$ to zero. Appendix~\ref{app:direct-obstruction}
gives a high-dimensional construction with $n_0<p$. For every
$k\in\operatorname{supp}(\beta^\star)$, the iterates obey
\begin{equation}
 \left|\widehat w_{t,k}^{\mathrm{dir}}\right|
 =\frac{1}{t}\max\!\left\{
   (t-1)\left|\widehat w_{t-1,k}^{\mathrm{dir}}\right|
   +\left|\beta_k^\star\right|-t\lambda_t,\,0
 \right\}.
\end{equation}
Each batch contributes the fixed signal $|\beta_k^\star|$, whereas inherited
shrinkage enters through $t\lambda_t$. Hence, whenever $t\lambda_t$
eventually dominates the signal, every active coordinate reaches zero after
finitely many tasks and remains zero thereafter, even though the coefficient
is fixed and more data arrive. For example,
$\lambda_t\asymp\sqrt{\log p/(tn_0)}$ gives
$t\lambda_t\asymp\sqrt{t\log p/n_0}\to\infty$. Offline pooled Lasso instead
applies only its current threshold, which decreases with cumulative sample
size. The construction and derivation are in Appendix~\ref{app:direct-obstruction}.
\end{example}

\noindent
\begin{minipage}[t]{0.59\textwidth}
\vspace{0pt}
Proposition~\ref{thm:direct-accumulation} identifies the score discrepancy
created by direct-center transmission; Figure~\ref{fig:mechanism-noisy} shows
its statistical consequence under noisy, correlated, high-dimensional designs
($p>n_t$, Gaussian AR(1) covariates, $\sigma=1$): substantial attenuation of
the true active coefficients. By contrast, SP-CL remains aligned with
offline pooled Lasso throughout. Appendix~\ref{app:direct-obstruction} shows
that the score discrepancy is robust across the tested penalty constants,
whereas the severity of signal attenuation depends on the penalty level.
\end{minipage}\hfill
\begin{minipage}[t]{0.41\textwidth}
\vspace{0pt}
\centering
\includegraphics[width=\linewidth]{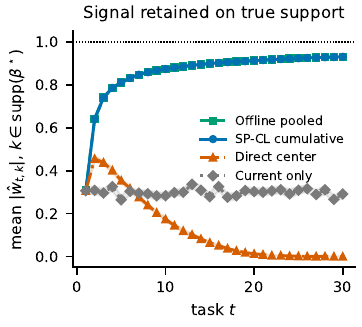}
\captionof{figure}{Signal retained on the true active support ($p=200$, $s=10$,
$n_t=40$, $T=30$; 50 streams).}
\label{fig:mechanism-noisy}
\end{minipage}
\par\medskip

\paragraph{Persistent and task-local first-order information.}
Proposition~\ref{thm:direct-accumulation} suggests separating estimation from
transmission: a regularizer may be useful for the current estimate without
belonging to the historical objective that the next state is intended to
reconstruct. In the present Lasso setting, the $\ell_1$ penalty remains
essential for sparse estimation, but its KKT contribution is not part of the
historical data loss. SP-CL therefore retains the $\ell_1$ penalty in
estimation but removes its KKT contribution before transmission -- regularize
locally, transmit the historical data score. Appendix~\ref{app:realdata}
further confirms this exact-equivalence mechanism on a real chronological
climate-prediction stream.

\subsection{Score-preserving continual Lasso}
\label{sec:score-preserving-update}

Accordingly, SP-CL removes the current $\ell_1$-KKT
contribution before transmitting the state. Initialize $H_0=0$, $u_0=0$, and
$N_0=0$.

\paragraph{Cumulative sparse fit.}
Given $(H_{t-1},u_{t-1})$ satisfying $H_{t-1}u_{t-1}=G_{t-1}$, set
$N_t=N_{t-1}+n_t$ and $H_t=H_{t-1}+X_t^{\mathsf T}X_t$, and solve
\begin{equation}
  \widehat w_t\in
  \argmin_w\left\{
    \frac12\norm{Y_t-X_tw}_2^2
    +\frac12\norm{w-u_{t-1}}_{H_{t-1}}^2
    +N_t\lambda_{w,t}\norm{w}_1
  \right\}.
  \label{eq:sp-stage-one}
\end{equation}

\paragraph{Score-preserving state correction.}
Choose $z_t\in\partial\norm{\widehat w_t}_1$. The KKT condition for
\eqref{eq:sp-stage-one} gives
\[
  S_t(\widehat w_t)=G_t-H_t\widehat w_t
  =N_t\lambda_{w,t}z_t.
\]
Unlike Proposition~\ref{thm:direct-accumulation}, this expression contains no
penalty terms accumulated from earlier tasks: its only regularization
contribution is the one used for the current sparse fit.
Define
\begin{equation}
  d_t=(H_t/N_t)^+z_t,
  \qquad
  u_t=\widehat w_t+\lambda_{w,t}d_t.
  \label{eq:sp-carrier}
\end{equation}
The KKT equation also implies $z_t\in\operatorname{Col}(H_t)$, so
$H_td_t=N_tz_t$ and therefore
\[
  S_t(u_t)=G_t-H_tu_t=0.
\]
Thus, $\ell_1$ regularization acts on the estimator $\widehat w_t$, while the
transmitted quadratic state retains the historical curvature and data score
without accumulated penalty contributions. The Moore--Penrose inverse selects
the minimum-norm correction $d_t$ when $H_t$ is singular.

The full-history sparse benchmark is the offline pooled Lasso
\begin{equation}
  \widehat w_t^{\mathrm{off}}
  \in\arg\min_w\left\{
    \frac{1}{2N_t}\sum_{j=1}^t\norm{Y_j-X_jw}_2^2
    +\lambda_{w,t}\norm{w}_1
  \right\}.
  \label{eq:offline-pooled}
\end{equation}
It is unavailable to a continual learner because it requires the historical
raw batches.

\begin{theorem}[Exact pooled-Lasso equivalence]
\label{thm:exact}
For every $t$, the cumulative estimator $\widehat w_t$ in
\eqref{eq:sp-stage-one} is a minimizer of the offline pooled Lasso
\eqref{eq:offline-pooled}.
\end{theorem}

The proof is given in Appendix~\ref{app:method-proofs}. Hence SP-CL exactly
reproduces the pooled-Lasso objective without revisiting historical data; only
$H_tu_t=G_t$ matters, so $u_t$ represents the historical linear score.

\paragraph{Current-task correction under task shift.}
When task coefficients differ, the cumulative estimator targets a pooled
coefficient rather than the current-task coefficient. Motivated by the
two-stage correction principle in high-dimensional transfer learning
\citep{li2022transfer,tianfeng2023}, we use the current batch to sparsely
correct the pooled estimate without modifying the transmitted state:
\begin{equation}
 \widehat\Delta_t=
 \argmin_\Delta\left\{
 \frac{1}{2n_t}\norm{Y_t-X_t(\widehat w_t+\Delta)}_2^2+
 \lambda_{\Delta,t}\norm{\Delta}_1\right\},
 \qquad
 \widehat\beta_t=\widehat w_t+\widehat\Delta_t.
 \label{eq:target-correction}
\end{equation}
The predictor $\widehat\beta_t$ is used for current-task prediction but is not
written back into the state. Section~\ref{sec:theory} analyzes when this
correction improves on current-only and cumulative estimation.

\begin{algorithm}[H]
\caption{Score-preserving continual Lasso (SP-CL)}
\label{alg:spcl}
\begin{algorithmic}[1]
\Require Sequential batches $\{(X_t,Y_t)\}_{t=1}^T$ and penalties
$\{\lambda_{w,t},\lambda_{\Delta,t}\}_{t=1}^T$
\State Initialize $H_0=0$, $u_0=0$, and $N_0=0$
\For{$t=1,\ldots,T$}
  \State \textbf{Cumulative sparse fit:} Set $N_t=N_{t-1}+n_t$ and solve \eqref{eq:sp-stage-one} for $\widehat w_t$
  \State \textbf{Score-preserving state update:} Set $H_t=H_{t-1}+X_t^{\mathsf T}X_t$ and construct $u_t$ by \eqref{eq:sp-carrier}
  \State \textbf{Current-task correction (under task shift):} Solve \eqref{eq:target-correction} and output $\widehat\beta_t$
\EndFor
\end{algorithmic}
\end{algorithm}

Only $(H_t,u_t,N_t)$ is transmitted to task $t+1$; neither
$\widehat\beta_t$ nor $\widehat\Delta_t$ is written into the transmitted state.
In unrestricted designs, this requires a $p\times p$ curvature state; under a
fixed known common eigenbasis for pairwise commuting empirical task Hessians,
Corollary~\ref{cor:commuting-state} gives an exact $O(p)$-dimensional
transmitted state without changing the original $\ell_1$ penalty.

\subsection{Consequence Under a Shared Coefficient}
\label{sec:shared}

The simplest consequence of Theorem~\ref{thm:exact} arises when
$\beta_1^\star=\cdots=\beta_t^\star=\beta^\star$: task mismatch disappears and
the cumulative and current targets coincide.

\begin{corollary}[Full-history sparse estimation under a shared coefficient]
\label{cor:shared-rate}
Suppose $\beta_1^\star=\cdots=\beta_t^\star=\beta^\star$,
$\norm{\beta^\star}_0\le s$, and the pooled design and noise satisfy the
standard sub-Gaussian conditions with well-conditioned covariance. If
$N_t\gtrsim s\log p$ and $\lambda_{w,t}\asymp\sqrt{\log p/N_t}$, then, with
high probability,
\begin{equation}
  \norm{\widehat w_t-\beta^\star}_{\Sigma}^2
  \lesssim\frac{s\log p}{N_t},
  \qquad
  \norm{\widehat w_t-\beta^\star}_1
  \lesssim s\sqrt{\frac{\log p}{N_t}}.
  \label{eq:shared-rate}
\end{equation}
Let $S=\operatorname{supp}(\beta^\star)$. Under the additional
irrepresentability, beta-min, and sample-size conditions in
Appendix~\ref{app:shared-selection},
\begin{equation}
  \Pbb\!\left\{
    \operatorname{supp}(\widehat w_t)=S,\quad
    \operatorname{sgn}(\widehat w_{t,S})=
    \operatorname{sgn}(\beta_S^\star)
  \right\}\longrightarrow 1.
  \label{eq:shared-selection}
\end{equation}
\end{corollary}

Proofs are in Appendices~\ref{app:method-proofs} and~\ref{app:shared-selection};
unlike Table~\ref{tab:shared_sparse_comparison}'s rate comparison,
Theorem~\ref{thm:exact} gives exact objective-level equivalence.

\begin{table}[H]
\centering
\caption{Statistical comparison in the shared-coefficient setting. Rates show
representative squared-estimation/excess-prediction scales under
well-conditioned designs; dense and sparse rows correspond to different
parameter classes and are not direct minimax comparisons.}
\label{tab:shared_sparse_comparison}
\small
\setlength{\tabcolsep}{4.5pt}
\renewcommand{\arraystretch}{1.1}
\begin{tabular}{@{}lcccc@{}}
\toprule
Estimator & Raw history & Structure & Error scale & Signed support \\
\midrule
Current-only Lasso
& No & Sparse & $\displaystyle \frac{s\log p}{n_t}$ & $\checkmark^{\ddagger}$ \\
Offline pooled Lasso
& Yes & Sparse & $\displaystyle \frac{s\log p}{N_t}$ & $\checkmark^{\ddagger}$ \\
Gen.\ $\ell_2$ CL \citep{zhao2024}
& No & Dense & $\displaystyle \frac{p}{N_t}^{\dagger}$ & -- \\
\textbf{SP-CL cumulative}
& No & Sparse & $\displaystyle \frac{s\log p}{N_t}$ & $\checkmark^{\ddagger}$ \\
\bottomrule
\end{tabular}

\vspace{2pt}
\footnotesize
$\dagger$ Representative scaling of the oracle-order generalized-$\ell_2$
error in a common, well-conditioned shared-coefficient specialization.
$\ddagger$ Under the additional irrepresentability, beta-min, and sample-size
conditions of Corollary~\ref{cor:shared-rate}, or their current-only analogue.
\end{table}

\section{Current-Task Transfer under Related Task Shift}
\label{sec:theory}

\subsection{Design condition and the related-task class}
\label{sec:oracle-regime}

We now allow task-specific coefficients and ask when related history improves
current-task estimation under the following design condition.

\begin{condition}[Design and noise]
\label{cond:design}
The rows $x_{ji}$ are independent mean-zero sub-Gaussian vectors with common
covariance $\Sigma$, whose eigenvalues lie in $[\kappa_l,\kappa_u]$ for fixed
constants $0<\kappa_l\le\kappa_u<\infty$; the errors $\varepsilon_{ji}$ are
independent mean-zero sub-Gaussian with uniformly bounded sub-Gaussian norm.
\end{condition}

For each task $t\ge2$, consider the oracle related-task class
\[
 \Theta_t(s,h)=\Big\{(\beta_1^\star,\ldots,\beta_t^\star):
 \norm{\beta_t^\star}_0\le s,\ \max_{j<t}\norm{\beta_j^\star-\beta_t^\star}_1\le h\Big\}.
\]
Here $h$ measures the largest $\ell_1$ discrepancy between the current task
and a historical task; smaller $h$ means more transferable history. We work
in an oracle related-task regime in which all historical tasks are informative
for the current target.

Under the common covariance in Condition~\ref{cond:design}, with
$\pi_{j,t}=n_j/N_t$, the cumulative population target is
$w_t^\star=\sum_{j=1}^t\pi_{j,t}\beta_j^\star$. Hence, writing
$\Delta_t^\star=\beta_t^\star-w_t^\star$, the relatedness condition gives
$\norm{\Delta_t^\star}_1\le h$.

\subsection{Current-task transfer}
\label{sec:three-estimators}

We compare the cumulative and adapted predictors with the current-only Lasso
$\widehat\beta_t^{\mathrm{cur}}$, obtained by minimizing
$(2n_t)^{-1}\norm{Y_t-X_tb}_2^2+\lambda_{\mathrm{cur},t}\norm{b}_1$ over $b$.
The cumulative predictor $\widehat w_t$ uses all historical information
through the score-preserving state, and the adapted predictor
$\widehat\beta_t=\widehat w_t+\widehat\Delta_t$ uses the current batch to
correct cumulative mismatch. Let $D_t=\norm{w_t^\star-\beta_t^\star}_\Sigma$.

\begin{theorem}[Anytime estimation under related task shift]
\label{thm:transfer}
Fix $\delta\in(0,1)$ and, for $t\ge1$, let
$L_{t,\delta}=\log\!\left(C_{\mathrm{u}}p\,t(t+1)/\delta\right)$, where
$C_{\mathrm{u}}$ is sufficiently large, and take
$\lambda_{\mathrm{cur},t},\lambda_{\Delta,t}\asymp
\sqrt{L_{t,\delta}/n_t}$ and
$\lambda_{w,t}\asymp\sqrt{L_{t,\delta}/N_t}$. Under
Condition~\ref{cond:design} and the sparsity and sample-size conditions below,
with probability at least $1-\delta$, the following statements hold
simultaneously for every $t\ge1$ for which their respective conditions hold.
For \emph{(a)} current-only fitting, suppose $\norm{\beta_t^\star}_0\le s$
and $n_t\gtrsim sL_{t,\delta}$. Then
\[
  \norm{\widehat\beta_t^{\mathrm{cur}}-\beta_t^\star}_{\Sigma}^2
  \lesssim\frac{sL_{t,\delta}}{n_t},
  \qquad
  \norm{\widehat\beta_t^{\mathrm{cur}}-\beta_t^\star}_1
  \lesssim s\sqrt{\frac{L_{t,\delta}}{n_t}}.
\]
For \emph{(b)} the cumulative generalist, suppose
$\norm{w_t^\star}_0\le s_{w,t}$ and $N_t\gtrsim s_{w,t}L_{t,\delta}$. Then
\[
  \norm{\widehat w_t-w_t^\star}_{\Sigma}^2
  \lesssim\frac{s_{w,t}L_{t,\delta}}{N_t},
  \qquad
  \norm{\widehat w_t-w_t^\star}_1
  \lesssim s_{w,t}\sqrt{\frac{L_{t,\delta}}{N_t}}.
\]
Moreover, $\risk_t(\widehat w_t)\lesssim
s_{w,t}L_{t,\delta}/N_t+D_t^2$.

For \emph{(c)} the adapted specialist, let $t\ge2$ and suppose
$(\beta_1^\star,\ldots,\beta_t^\star)\in\Theta_t(s,h)$ with
$h\sqrt{L_{t,\delta}/n_t}=o(1)$ along the underlying high-dimensional
sequence, $h\lesssim\sqrt s$,
$n_t\gtrsim L_{t,\delta}$, and $N_{t-1}\gtrsim sL_{t,\delta}$. Then
\begin{equation}
\begin{aligned}
  \norm{\widehat\beta_t-\beta_t^\star}_{\Sigma}^2
  &\lesssim\frac{sL_{t,\delta}}{N_t}
  +\left(h\sqrt{\frac{L_{t,\delta}}{n_t}}\wedge h^2\right),
  &\norm{\widehat\beta_t-\beta_t^\star}_1
  &\lesssim s\sqrt{\frac{L_{t,\delta}}{N_t}}+h.
\end{aligned}
\label{eq:oracle-rate}
\end{equation}
\end{theorem}

For polynomially growing $t$ and $\delta^{-1}$,
$L_{t,\delta}\asymp\log p$; stronger beta-min and support-alignment conditions
yield signed-support recovery (Appendix~\ref{app:selection}). By
Theorem~\ref{thm:exact}, the pooled first stage is exactly its offline
counterpart, so the current correction inherits the offline two-stage rate
without revisiting history; Appendix~\ref{app:transfer-proof} gives the full
constant-level statement.

\paragraph{Rate interpretation.}
Writing $q_t=\sqrt{L_{t,\delta}/n_t}$, current-only fitting has scale $sq_t^2$,
whereas the specialist has scale $sL_{t,\delta}/N_t+(hq_t\wedge h^2)$. More
history reduces the first term, while larger task discrepancy increases the
second. Under common covariance, the generalist instead pays direct mismatch
$D_t^2\lesssim h^2$, so correction is most useful when pooling reduces variance
but introduces an estimable sparse mismatch.

\section{Continual Evaluation under Task Heterogeneity}
\label{sec:metrics}

We now compare the cumulative generalist $b_t^{\mathrm G}=\widehat w_t$ and
current-adapted specialist $b_t^{\mathrm S}=\widehat\beta_t$ (superscripts
$\mathrm G$/$\mathrm S$ throughout) through Overall, FWT, and BWT under task
heterogeneity.

\paragraph{Task conflict.}
Under the common covariance, define the unavoidable single-head conflict
$\mathsf{TC}_t:=\sum_{j=1}^t\pi_{j,t}
\norm{\beta_j^\star-w_t^\star}_{\Sigma}^2$. Then, for every $b\in\R^p$,
\begin{equation}
  \overline\mse_t(b)
  =
  \sum_{j=1}^t\pi_{j,t}\mse_j(\beta_j^\star)
  +\mathsf{TC}_t+\norm{b-w_t^\star}_{\Sigma}^2.
  \label{eq:conflict-decomposition}
\end{equation}
Thus $\mathsf{TC}_t$ is the unavoidable loss from representing heterogeneous
tasks by one coefficient, while the final term is the avoidable deviation from
the best cumulative single head. The analogous historical decomposition is
centered at $w_{t-1}^\star$ with conflict $\mathsf{TC}_{t-1}$.
Proposition~\ref{prop:single-head} in Appendix~\ref{app:single-head-conflict}
gives a two-task lower bound.

\refstepcounter{proposition}
\noindent\textbf{Proposition~\theproposition\ (Generalist--specialist metric
relations).}\label{prop:metric-decomposition}
\emph{Let $\overline\mse_{<t}(b)$ denote the weighted average over tasks $j<t$;
define the specialist's historical cost
$\mathrm{BC}_t:=\overline\mse_{<t}(\widehat\beta_t)-
\overline\mse_{<t}(\widehat w_t)$ and current gain
$\mathrm{CG}_t:=\mse_t(\widehat w_t)-\mse_t(\widehat\beta_t)$.} Then, for the
generalist and specialist,
\textup{(a)} $\mathrm{FWT}_t^{\mathrm S}-\mathrm{FWT}_t^{\mathrm G}
=(t-1)^{-1}\sum_{k=2}^t\mathrm{CG}_k$ and
\textup{(b)} $\overline\mse_t(\widehat\beta_t)-\overline\mse_t(\widehat w_t)
=N_{t-1}\mathrm{BC}_t/N_t-n_t\mathrm{CG}_t/N_t$. Under equal batch sizes, for
any deployed sequence $\{\widehat b_j\}$,
\textup{(c)} $\mathrm{BWT}_t(\widehat b)
=(t-1)^{-1}\sum_{j<t}\risk_j(\widehat b_j)-\mathsf{TC}_{t-1}
-\norm{\widehat b_t-w_{t-1}^\star}_{\Sigma}^2$.
The routine proof is given in Appendix~\ref{app:metric-proof}.

Part~(a) links current-task correction directly to the specialist's FWT gain,
while part~(b) expresses Overall performance as a current-gain versus
historical-cost trade-off. Part~(c) additionally isolates task conflict: even
with perfect historical information, $\widehat b_j=\beta_j^\star$ for $j<t$
and $\widehat b_t=w_{t-1}^\star$ give
$\mathrm{BWT}_t(\widehat b)=-\mathsf{TC}_{t-1}$. Thus old-task degradation need
not reflect loss of transmitted information; Appendix~\ref{app:partial-correction}
gives a budgeted interpolation between the two outputs.
\section{Experiments}
\label{sec:experiments}

\paragraph{Setup.} We evaluate sparse estimation, task-shift transfer, and
continual-learning performance under model~\eqref{eq:model} with Gaussian
AR(1) covariates $\Sigma_{k\ell}=0.5^{|k-\ell|}$, $\sigma=1$, and sparse
coefficients. We report analytic population risks over 50 streams; tuning,
solver, and sensitivity details are in Appendix~B.

\paragraph{Sparse high-dimensional estimation.}
We first consider a shared $s=10$ sparse coefficient with $p=1000$,
$n_t=40$, and $T=20$, so that $N_T=800<p$ throughout the stream.
Figure~\ref{fig:highdim-sparse} compares sparse and dense estimators. At the
final task, SP-CL attains excess risk $0.115$, compared with $2.92$ for
pilot-tuned pooled Ridge, while using about $20$ effective coordinates rather
than roughly $950$--$990$. In this globally underdetermined regime, SP-CL retains the statistical benefit of pooling while the dense pooled estimator remains substantially less accurate.

\begin{figure}[H]
  \centering
  \includegraphics[width=0.86\textwidth]{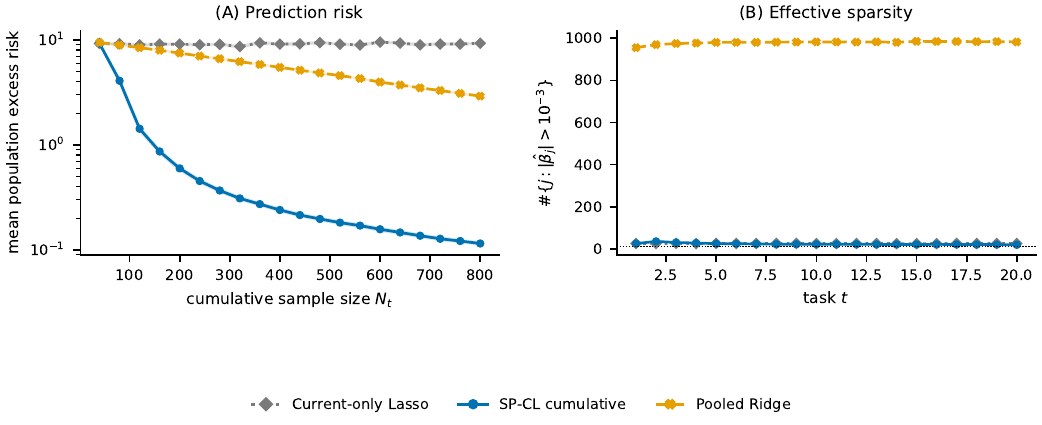}
  \caption{Sparse efficiency ($p=1000$, $s=10$, $n_t=40$, $T=20$,
  $N_T=800<p$; 50 streams). Excess risk and support size; pilot-tuned Ridge
  versus SP-CL (pooled Lasso matches SP-CL).}
  \label{fig:highdim-sparse}
\end{figure}
\FloatBarrier

\paragraph{Transfer under task shift.}
We next vary the task mismatch $h$ and the amount of historical information
$m=N_{t-1}/n_t$ under a fixed-support coefficient path
(Appendix~\ref{app:task-shift-simulation}). Figure~\ref{fig:transfer-geometry}
reports the gain from pooling, the additional gain from current-task
correction, and their total gain over current-only fitting. Pooling is
beneficial for small $h$; at intermediate mismatch, correction can rescue
negative transfer from pooling; and with severe shift and little history, even
the corrected estimator can underperform current-only fitting. The transition across \(h\) and \(m\) reflects the two terms in Theorem~\ref{thm:transfer}: more history reduces estimation error, whereas larger task mismatch increases the correction cost.

\begin{figure}[!htbp]
  \centering
  \includegraphics[width=0.90\textwidth]{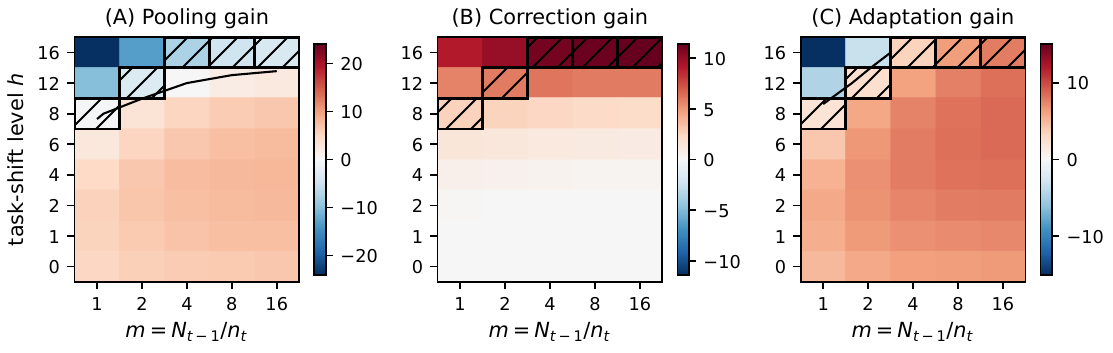}
  \caption{Transfer under task shift (50 streams). Mean gains as task mismatch
  $h$ and history ratio $m=N_{t-1}/n_t$ vary:
  $G_{\mathrm{pool}}=\risk_t(\widehat\beta_t^{\mathrm{cur}})-\risk_t(\widehat w_t)$,
  $G_{\mathrm{corr}}=\risk_t(\widehat w_t)-\risk_t(\widehat\beta_t)$, and
  $G_{\mathrm{adapt}}=G_{\mathrm{pool}}+G_{\mathrm{corr}}$; positive values
  $=$ improvement and hatching $=$ correction rescue
  (Appendix~\ref{app:task-shift-simulation}).}
  \label{fig:transfer-geometry}
\end{figure}
\FloatBarrier

\paragraph{Continual-learning performance.}
We fix $T=20$ and vary $h\in\{0,4,8,12\}$. Figure~\ref{fig:cl-metrics}
compares the cumulative generalist $\widehat w_t$ with the current-adapted
specialist $\widehat\beta_t$ on final-task Overall, FWT, and BWT. At $h=12$,
the specialist's forward transfer exceeds the generalist's
($\mathrm{FWT}^{\mathrm S}=8.23$ vs.\ $\mathrm{FWT}^{\mathrm G}=6.45$), while
its backward transfer falls well below the generalist's
($\mathrm{BWT}^{\mathrm S}=-3.23$ vs.\ $\mathrm{BWT}^{\mathrm G}=+0.98$);
Overall favors the generalist throughout. These results expose the
generalist--specialist trade-off: increasing specialization improves forward
transfer but incurs a growing historical cost, as decomposed in
Proposition~\ref{prop:metric-decomposition}.

\begin{figure}[!htbp]
  \centering
  \includegraphics[width=\textwidth]{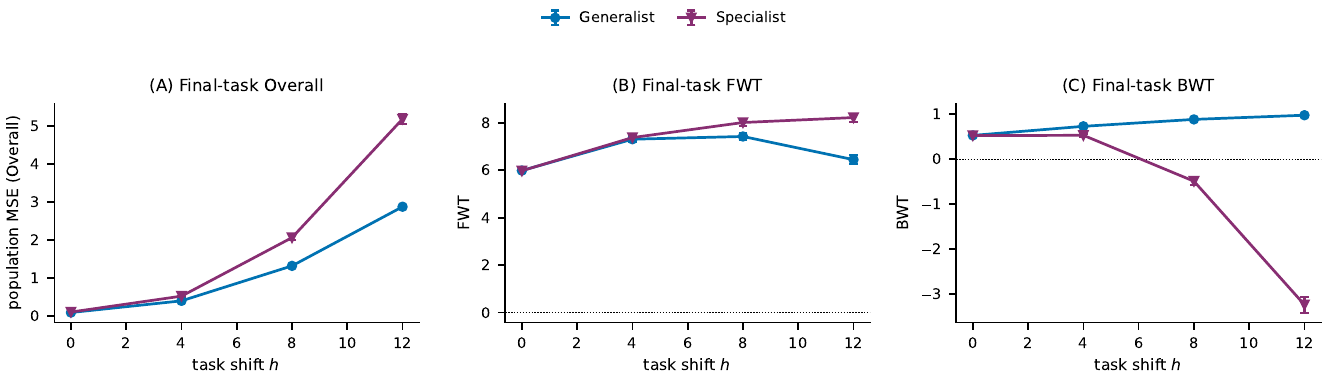}
  \caption{Continual performance under task heterogeneity ($p=200$, $s=10$,
  $n_t=40$, $T=20$; 50 streams). Final-task Overall, FWT, and BWT for the
  cumulative generalist and current-adapted specialist as $h$ varies; error
  bars show one standard error.}
  \label{fig:cl-metrics}
\end{figure}
\FloatBarrier

\section{Conclusion}

We show that recursively transmitting an $\ell_1$-penalized estimate propagates
past shrinkage, and develop SP-CL to separate sparse estimation from state
transmission. Its cumulative estimator exactly matches offline pooled Lasso,
while a current-task correction provides statistically efficient adaptation
under task shift. Overall/FWT/BWT decompositions further separate information
retention, task conflict, and specialization. SP-CL therefore treats regularization used for estimation and information retained for future tasks as distinct objects.

\section*{Reproducibility statement}
The algorithm, assumptions, and complete proofs are provided in the main paper and appendices. The supplementary ZIP contains anonymized core implementation, fixed seeds, and scripts/configurations for reproducing the main experimental results. A cleaned and documented version of the full codebase will be released publicly after the review process.

\section*{AI use statement}
Generative AI tools were used to help develop the conceptual framework,
formulate and organize mathematical claims, design the initial experimental
protocol, and draft and edit this manuscript.  The authors manually reviewed
all AI-assisted content, checked the cited literature against primary sources,
and take responsibility for the final text, mathematical claims, and empirical
artifacts.

\bibliography{references}
\bibliographystyle{iclr2027_conference}

\appendix
\section{Proofs for Score-Preserving Continual Lasso}
\label{app:method-proofs}

\subsection{Quadratic closure}

\begin{proof}[Proof of Lemma~\ref{lem:quadratic-closure}]
Let $\widetilde X_t=(X_1^{\mathsf T},\ldots,X_t^{\mathsf T})^{\mathsf T}$
and $\widetilde Y_t=(Y_1^{\mathsf T},\ldots,Y_t^{\mathsf T})^{\mathsf T}$.
Then $H_t=\widetilde X_t^{\mathsf T}\widetilde X_t$ and
$G_t=\widetilde X_t^{\mathsf T}\widetilde Y_t$. Since
$\operatorname{Col}(\widetilde X_t^{\mathsf T}\widetilde X_t)=
\operatorname{Col}(\widetilde X_t^{\mathsf T})$, we have
$G_t\in\operatorname{Col}(H_t)$, so there exists $u_t$ satisfying
$H_tu_t=G_t$. Expanding the historical squared loss gives
$\frac12 w^{\mathsf T}H_tw-w^{\mathsf T}G_t+C_t'$. Completing the square
then yields \eqref{eq:quadratic-closure}.
\end{proof}

\subsection{Penalty residual under direct transmission of the penalized center}

\begin{proof}[Proof of Proposition~\ref{thm:direct-accumulation}]
The KKT condition for \eqref{eq:direct-stage} gives
\[
  H_t\widehat w_t^{\mathrm{dir}}
  =X_t^{\mathsf T}Y_t
   +H_{t-1}\widehat w_{t-1}^{\mathrm{dir}}
   -N_t\lambda_tz_t^{\mathrm{dir}}.
\]
Subtracting this identity from $G_t=G_{t-1}+X_t^{\mathsf T}Y_t$ yields
$D_t=D_{t-1}+N_t\lambda_tz_t^{\mathrm{dir}}$ for
$D_t=G_t-H_t\widehat w_t^{\mathrm{dir}}$. Iterating proves
\eqref{eq:direct-accumulation}.
\end{proof}

\subsection{Exact pooled-Lasso equivalence}

\begin{proof}[Proof of Theorem~\ref{thm:exact}]
The KKT equation for \eqref{eq:sp-stage-one} and the induction hypothesis
$H_{t-1}u_{t-1}=G_{t-1}$ give
\[
  H_t\widehat w_t=G_t-N_t\lambda_{w,t}z_t.
\]
Because $(H_t/N_t)d_t=z_t$ on $\operatorname{Col}(H_t)$,
$H_t(\lambda_{w,t}d_t)=N_t\lambda_{w,t}z_t$, proving
$H_tu_t=G_t$. Lemma~\ref{lem:quadratic-closure} then shows that the smooth
part of \eqref{eq:sp-stage-one} is, up to a constant, the full-history squared
loss. Dividing by $N_t$ yields \eqref{eq:offline-pooled}.
\end{proof}

\subsection{Exact compressed state under commuting empirical task Hessians}
\label{app:commuting-state}

\begin{corollary}[Linear-dimensional transmitted state under commuting empirical
task Hessians]
\label{cor:commuting-state}
For each batch, let
$\widehat\Sigma_j=n_j^{-1}X_j^{\mathsf T}X_j$. Suppose that the symmetric
matrices $\{\widehat\Sigma_j\}_{j\ge1}$ commute pairwise and have a fixed
known common orthogonal eigenbasis $U$, so that
\[
  \widehat\Sigma_j=U\Gamma_jU^{\mathsf T},
  \qquad
  \Gamma_j\ \text{diagonal}.
\]
Set $D_0=0$, $a_0=0$, and, at task $t$, update
\[
  D_t=D_{t-1}+n_t\Gamma_t.
\]
The cumulative fit and score-preserving update can be computed as
\begin{align}
  \widehat w_t\in\argmin_w\bigg\{
    &\frac12\norm{Y_t-X_tw}_2^2
    +\frac12(U^{\mathsf T}w-a_{t-1})^{\mathsf T}
      D_{t-1}(U^{\mathsf T}w-a_{t-1}) \notag\\
    &\hspace{35mm}+N_t\lambda_{w,t}\norm{w}_1
  \bigg\},                                                   \label{eq:compressed-stage}\\
  a_t&=U^{\mathsf T}\widehat w_t+
  N_t\lambda_{w,t}D_t^+U^{\mathsf T}z_t,\qquad
  z_t\in\partial\norm{\widehat w_t}_1.                       \label{eq:compressed-carrier}
\end{align}
For every $t$, $D_ta_t=U^{\mathsf T}G_t$, and
\eqref{eq:compressed-stage} has the same minimizer set as the full-state
SP-CL update and the offline pooled Lasso in \eqref{eq:offline-pooled}.
Thus $(D_t,a_t,N_t)$ is a $2p+1$-dimensional transmitted state. The
$\ell_1$ penalty remains $\norm{w}_1$ in the original parameterization.
This is an $O(p)$ dynamic-state representation conditional on the shared basis
$U$; if a dense $U$ must also be stored, total persistent storage remains
$O(p^2)$.
\end{corollary}

\begin{proof}
Because the empirical task Hessians are symmetric and commute, they are
simultaneously diagonalizable, giving the displayed representation. Let
$g_t=U^{\mathsf T}G_t$. Then
\[
  H_t=\sum_{j\le t}X_j^{\mathsf T}X_j=UD_tU^{\mathsf T},
  \qquad
  g_t=g_{t-1}+U^{\mathsf T}X_t^{\mathsf T}Y_t.
\]
We prove $D_ta_t=g_t$ by induction. It holds at $t=0$. If it holds at
$t-1$, then the quadratic historical term in \eqref{eq:compressed-stage}
equals, up to a constant,
\[
  \frac12w^{\mathsf T}UD_{t-1}U^{\mathsf T}w
  -w^{\mathsf T}UD_{t-1}a_{t-1}
  =
  \frac12w^{\mathsf T}H_{t-1}w-w^{\mathsf T}G_{t-1}.
\]
Hence \eqref{eq:compressed-stage} is exactly
\eqref{eq:sp-stage-one}. Its KKT equation, after left multiplication by
$U^{\mathsf T}$, is
\[
  D_tU^{\mathsf T}\widehat w_t
  =g_t-N_t\lambda_{w,t}U^{\mathsf T}z_t.
\]
As in the proof of Theorem~\ref{thm:exact},
$z_t\in\operatorname{Col}(H_t)$; consequently
$U^{\mathsf T}z_t\in\operatorname{Col}(D_t)$. Applying
\eqref{eq:compressed-carrier} therefore gives
\[
  D_ta_t
  =D_tU^{\mathsf T}\widehat w_t+
    N_t\lambda_{w,t}D_tD_t^+U^{\mathsf T}z_t
  =g_t.
\]
The induction closes, so the reduced quadratic term is identical to the
full-state one at every task. Theorem~\ref{thm:exact} then yields the stated
offline pooled-Lasso equivalence.
\end{proof}

\subsection{Proof of the estimation claim in Corollary~\ref{cor:shared-rate}}

\begin{proof}
Theorem~\ref{thm:exact} identifies $\widehat w_t$ with an offline pooled
Lasso minimizer targeting the $s$-sparse $\beta^\star$. The stated result is
therefore the standard pooled Lasso bound under the displayed design, noise,
and sample-size conditions, including its usual prediction and $\ell_1$ error
bounds.
\end{proof}

\section{Supporting results for direct transmission of the penalized center}
\label{app:direct-obstruction}

\subsection{Persistent-sign lower bound}

This appendix gives detailed consequences of the penalty residual in
Proposition~\ref{thm:direct-accumulation}, including the underdetermined
construction used in Example~\ref{ex:repeated-shrinkage}.
Suppose $n_j=n_0$ and $\beta_j^\star=\beta^\star$ for all $j$, write
$\lambda_j=\lambda_{w,j}$, and define
\[
  q_j=\frac{1}{n_0}X_j^{\mathsf T}Y_j,\qquad
  \xi_j=q_j-\frac{1}{n_0}X_j^{\mathsf T}X_j\beta^\star,\qquad
  \xi_{1:t}=\sum_{j=1}^t\xi_j.
\]
Then $G_t=H_t\beta^\star+n_0\xi_{1:t}$. Let
$S=\operatorname{supp}(\beta^\star)$, $s=|S|$, and
$\overline H_t=H_t/(tn_0)$.

\begin{proposition}[Persistent-sign lower bound]
\label{prop:persistent-sign}
Suppose that, through time $t$,
\[
  \operatorname{supp}(\widehat w_j^{\mathrm{dir}})=S,
  \qquad
  z_{j,S}^{\mathrm{dir}}=\operatorname{sgn}(\beta_S^\star),
  \qquad 1\le j\le t.
\]
Then
\begin{equation}
  \overline H_{t,SS}
  (\widehat w_{t,S}^{\mathrm{dir}}-\beta_S^\star)
  =
  \frac{1}{t}\xi_{1:t,S}
  -L_t\operatorname{sgn}(\beta_S^\star),
  \qquad
  L_t=\frac{1}{t}\sum_{j=1}^t j\lambda_j.
  \label{eq:persistent-sign-identity}
\end{equation}
Consequently, if $\overline H_{t,SS}\succ0$ and
$\lambda_{\max}(\overline H_{t,SS})\le C_{\max}$, then
\begin{equation}
  \norm{\widehat w_t^{\mathrm{dir}}-\beta^\star}_2
  \ge
  \frac{
    \left[
      \sqrt{s}L_t-\norm{\xi_{1:t,S}/t}_2
    \right]_+
  }{C_{\max}}.
  \label{eq:persistent-sign-lower-bound}
\end{equation}
\end{proposition}

\begin{proof}
Divide \eqref{eq:direct-accumulation} by $tn_0$, use
$G_t=H_t\beta^\star+n_0\xi_{1:t}$, and restrict to $S$ to obtain
\eqref{eq:persistent-sign-identity}. The bound follows from the triangle
inequality and the upper spectral bound on $\overline H_{t,SS}$.
\end{proof}

For the nominal high-dimensional schedule
$\lambda_j\asymp\sqrt{\log p/(jn_0)}$,
$L_t\asymp\sqrt{t\log p/n_0}$. On the usual score-concentration event,
persistent signed-support recovery and $\ell_2$ consistency are incompatible
whenever $st\log p/n_0\not\to0$. This obstruction involves only the active
block; the full cumulative Gram matrix may remain singular.

\subsection{Support extinction}

\begin{example}[Support extinction under a batchwise underdetermined design]
\label{ex:support-extinction}
Suppose $s<n_0<p$ and, after ordering coordinates by
$S=\operatorname{supp}(\beta^\star)$, each noiseless batch satisfies
\[
  \frac{1}{n_0}X_t^{\mathsf T}X_t=
  \begin{pmatrix}
    I_s&0\\
    0&B_t
  \end{pmatrix},
  \qquad
  B_t\succeq0,\quad
  \operatorname{rank}(B_t)\le n_0-s,\qquad
  Y_t=X_{t,S}\beta_S^\star.
\]
Initialize $\widehat w_0^{\mathrm{dir}}=0$. Then
$\widehat w_{t,S^c}^{\mathrm{dir}}=0$ for all $t$, and every $k\in S$ obeys
\begin{equation}
  \widehat w_{t,k}^{\mathrm{dir}}
  =
  \mathcal S_{\lambda_t}\left\{
    \frac{
      \beta_k^\star+(t-1)\widehat w_{t-1,k}^{\mathrm{dir}}
    }{t}
  \right\},
  \qquad
  \mathcal S_a(u)=\operatorname{sgn}(u)(|u|-a)_+.
  \label{eq:support-extinction-recurrence}
\end{equation}
Equivalently, for
$c_{t,k}=t|\widehat w_{t,k}^{\mathrm{dir}}|$,
\[
  c_{t,k}
  =
  \left[
    c_{t-1,k}+|\beta_k^\star|-t\lambda_t
  \right]_+.
\]
For
$\lambda_t=c_\lambda\sqrt{\log(ep)/(tn_0)}$, $t\lambda_t\to\infty$.
Hence, for every $k\in S$, there is a finite $\tau_k$ such that
$\widehat w_{t,k}^{\mathrm{dir}}=0$ for all $t\ge\tau_k$, and thus
$\widehat w_t^{\mathrm{dir}}\to0$. The construction remains batchwise
underdetermined for arbitrarily large $p$, because only the $s$-dimensional
active block is required to be orthogonal.
\end{example}

\begin{proof}
On the active coordinates, the current and accumulated quadratic terms in
\eqref{eq:direct-stage} reduce, after division by $tn_0$, to a squared loss
centered at
$(\beta_k^\star+(t-1)\widehat w_{t-1,k}^{\mathrm{dir}})/t$.
Soft thresholding therefore gives \eqref{eq:support-extinction-recurrence},
and multiplication by $t$ gives the recursion for $c_{t,k}$. On $S^c$, the
response is zero and the previous center is zero; the $\ell_1$ penalty selects
zero at every step. Since $t\lambda_t\to\infty$, the decrements
$t\lambda_t-|\beta_k^\star|$ eventually become positive and have a divergent
sum. Thus $c_{t,k}$ reaches zero in finite time and remains zero thereafter.
\end{proof}

\subsection{Additional state-transmission diagnostics}
\label{app:state-transmission-diagnostics}

This subsection reports the numerical diagnostics underlying
Figure~\ref{fig:mechanism-noisy} and its penalty-constant sensitivity, and is
the noisy, correlated-design counterpart to the deterministic construction
above. Across the 50-seed noisy experiment
($p=200$, $s=10$, $n_t=40$, $T=30$, AR(1) $\rho=0.5$, $\sigma=1$,
$c_\lambda=1$), every solve converged (KKT residual: max $5.4\times10^{-9}$,
median $1.6\times10^{-9}$, over $1500$ rows and four estimators). The
Proposition~\ref{thm:direct-accumulation} identity holds to a relative error
of at most $1.7\times10^{-8}$ (median $4.7\times10^{-15}$), and the
SP-CL/offline pooled-Lasso objective gap remains at floating-point precision
throughout (max $8.9\times10^{-16}$, median exactly $0$). The single largest
value in each of these diagnostics traces to one seed at task $5$, where the
pooled design was exactly square ($N_5=p=200$) and produced a near-singular
Gram matrix: this inflates only the carrier-reconstruction (pseudoinverse)
residual, to $4.0\times10^{-7}$ at that cell, not the pooled-objective
equivalence, which for that same seed and task remains at $2.5\times10^{-15}$.
An analogous but smaller effect (residuals up to about $10^{-7}$ in a small
number of near-singular pooled-design cells) appears in the task-shift
experiments of Section~\ref{sec:experiments}, for the same conditioning
reason and with no loss of objective equivalence there either.

\begin{figure}[t]
  \centering
  \includegraphics[width=0.55\textwidth]{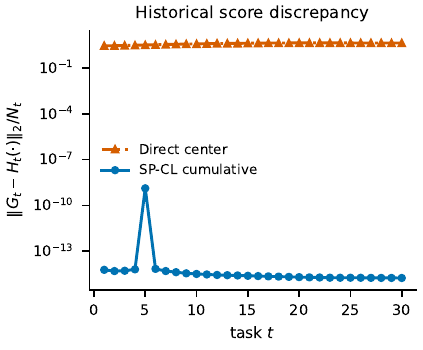}
  \caption{Historical score discrepancy
  $\lVert G_t-H_t(\cdot)\rVert_2/N_t$ under the same noisy, correlated
  design as Figure~\ref{fig:mechanism-noisy}, for direct-center ($\widehat
  w_t^{\mathrm{dir}}$) and SP-CL ($u_t$), log scale.}
  \label{fig:mechanism-noisy-score-gap}
\end{figure}

Figure~\ref{fig:mechanism-noisy-score-gap} plots this discrepancy directly:
SP-CL stays at floating-point precision at every task, while direct-center's
remains persistently positive and does not shrink as more data arrive.

At the final task, direct-center's mean population MSE is $10.99$, within
$0.1\%$ of the fully-shrunk null-predictor floor
($\sigma^2+\norm{\beta^\star}_{\Sigma}^2\approx11$), against $1.06$ for both
SP-CL and offline pooled; its mean true-positive rate on the $10$-coordinate
active support falls to $0.046$, against $1.000$ for SP-CL and offline
pooled.

\paragraph{Penalty-constant sensitivity.}
Repeating this experiment (same 50 seeds) for
$c_\lambda\in\{0.25,0.5,0.75,1,1.25\}$, Figure~\ref{fig:lambda-robustness}
shows that SP-CL's score discrepancy remains at floating-point precision at
every tested value, while direct-center's remains persistently positive
throughout the same range: the mechanism itself does not depend on the
penalty constant. The downstream severity does: at $c_\lambda=0.25$,
direct-center retains $66\%$ of the true coefficient magnitude on average
and its MSE ($2.2$) is far closer to SP-CL's ($1.06$) than to the null floor
($\approx11$); support collapse and near-null MSE are specific to
$c_\lambda\gtrsim0.75$ in this design. We do not claim direct-center
collapses for every statistically motivated penalty constant, only that it
does so robustly at and above the order used as the main setting.

\begin{figure}[t]
  \centering
  \includegraphics[width=\textwidth]{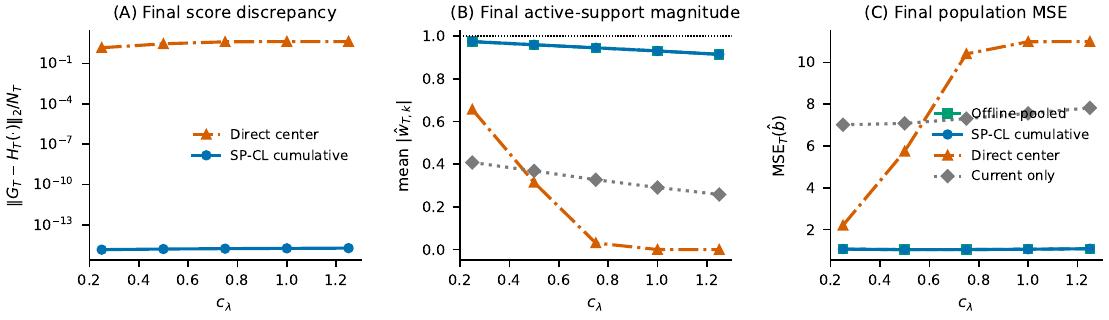}
  \caption{Final-task ($t=T$) sensitivity to the penalty constant $c_\lambda$,
  same 50 seeds reused across $c_\lambda\in\{0.25,0.5,0.75,1,1.25\}$. Panel
  (A): score discrepancy. Panel (B): active-support magnitude. Panel (C):
  population MSE.}
  \label{fig:lambda-robustness}
\end{figure}

\subsection{Additional sparse-efficiency results}
\label{app:sparse-efficiency}

Figure~\ref{fig:highdim-estimation} gives the squared estimation-error trace
for the high-dimensional experiment in Section~\ref{sec:experiments}. Its
ordering agrees with the risk comparison in Figure~\ref{fig:highdim-sparse}.
This experiment is not an exact variable-selection setting: the final-task
false discovery rate of SP-CL is approximately $0.47$. Its purpose is instead
to quantify the finite-sample prediction benefit of sparse estimation when
$p>N_T$.

\begin{figure}[t]
  \centering
  \includegraphics[width=0.58\textwidth]{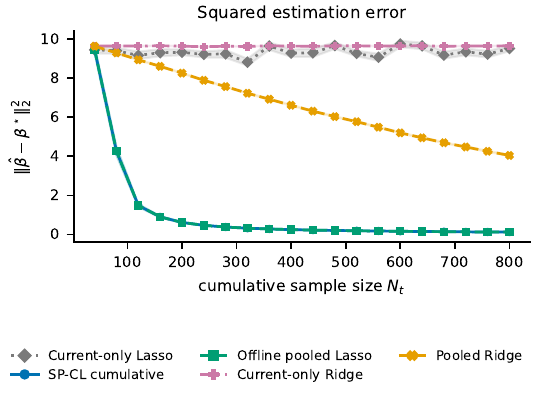}
  \caption{Squared $\ell_2$ estimation error in the high-dimensional
  sparse-versus-dense experiment of Figure~\ref{fig:highdim-sparse}. Pooled
  Ridge uses the pilot-tuned task-specific penalty schedule.}
  \label{fig:highdim-estimation}
\end{figure}

\subsection{Task-shift simulation construction}
\label{app:task-shift-simulation}

For Figure~\ref{fig:transfer-geometry}, we use $p=200$, $s=10$, and equal
batch sizes $n_t=40$. The fixed-support task path is
\[
  \beta_k^\star=\beta_0^\star+\frac{h(k-1)}{T-1}v,
\]
where $\norm{v}_1=1$ and $v$ is supported on five of the ten active
coordinates. Thus the final coefficient is within $\ell_1$ distance $h$ of
every earlier coefficient. We vary
$h\in\{0,1,2,4,6,8,12,16\}$ and
$m=N_{t-1}/n_t\in\{1,2,4,8,16\}$, using $T=m+1$ equal-size batches.
Each cell averages 50 common-random-number replicates, so the current-only
fit is shared across values of $m$ for a fixed $(h,\mathrm{seed})$.

\subsection{Unavoidable single-head conflict}
\label{app:single-head-conflict}

For completeness, any one coefficient must incur nonzero loss on at least one
of two tasks with distinct population optima:
\begin{proposition}[Unavoidable single-head conflict]
\label{prop:single-head}
For any two tasks $j<t$,
\begin{equation}
  \inf_{b\in\R^p}\max\{\risk_j(b),\risk_t(b)\}
  =\frac14\norm{\beta_j^\star-\beta_t^\star}_{\Sigma}^2.
  \label{eq:single-head-bound}
\end{equation}
\end{proposition}

\begin{proof}
The triangle inequality gives
$\norm{\beta_j^\star-\beta_t^\star}_{\Sigma}
\leq\norm{b-\beta_j^\star}_{\Sigma}+\norm{b-\beta_t^\star}_{\Sigma}$.
The lower bound follows after squaring, and equality is attained at
$b=(\beta_j^\star+\beta_t^\star)/2$.
\end{proof}

\section{Proof of Theorem~\ref{thm:transfer}}
\label{app:transfer-proof}

Parts~(a) and~(b) are the standard single-task and sparse-anchor Lasso
bounds, proved at the end of this section. We prove part~(c) by a
confidence-level specialization of \citet{tianfeng2023}; $c,C$ denote
constants that may change between lines and depend only on
$\kappa_l,\kappa_u$.

\textbf{Step 1 (Exact offline reduction).} By Theorem~\ref{thm:exact},
$\widehat w_t$ is exactly a minimizer of the offline pooled Lasso
\eqref{eq:offline-pooled} formed from $\mathcal D_1,\ldots,\mathcal D_t$,
with no dependence on how the inaccessible batches $\mathcal D_j$, $j<t$,
were summarized. Consequently the pair $(\widehat w_t,\widehat\Delta_t)$
defined by \eqref{eq:sp-stage-one} and \eqref{eq:target-correction} is
identical, as a random pair given $(\mathcal D_1,\ldots,\mathcal D_t)$, to
the output of Algorithm~1 (A-Trans-GLM) of \citet{tianfeng2023} specialized
to $\psi(u)=u^2/2$, run with target sample $\mathcal D_t$ and source set
$\{\mathcal D_j\}_{j<t}$: their transferring step is our
\eqref{eq:sp-stage-one} divided by $N_t$, and their debiasing step is our
\eqref{eq:target-correction}. This identification holds
exactly and is the only streaming-specific step in the proof; everything
below follows from their offline argument after making its
confidence level explicit.

\textbf{Step 2 (Notation map and population target).} Their target
parameter, target sample size, source total, and pooled total map to our
$\beta_t^\star$, $n_t$, $N_{t-1}$, and $N_t$; their source parameters map to
$\beta_j^\star$, $j<t$; their contrast bound maps to $h$. Under common
covariance, the population version of the transferring step solves
$\Sigma w=\sum_{j\le t}\pi_{j,t}\Sigma\beta_j^\star$, giving
$w_t^\star=\sum_{j\le t}\pi_{j,t}\beta_j^\star$, and
the related-task condition in Section~\ref{sec:oracle-regime} gives
$\norm{\beta_t^\star-w_t^\star}_1\le\sum_{j<t}\pi_{j,t}\norm{\beta_t^\star-\beta_j^\star}_1
\le h$, so $w_t^\star$ plays the role of their $w_{\mathcal A_h}$.

\textbf{Step 3 (Covariance compatibility).} Their Theorem~1 requires a
covariance-compatibility condition (their Assumption~4). Under common
covariance, $\widetilde\Sigma_h=\sum_k\alpha_k\Sigma=\Sigma$ and
$\Sigma_h^{(k)}=\Sigma$ for every historical task $k$, so
$\widetilde\Sigma_h^{-1}\Sigma^{(k)}=I_p$, whose induced $\ell_1$-operator
norm is $1<\infty$: the condition holds automatically, with no separate
assumption needed.

\textbf{Step 4 (Confidence-level transfer bound).} With $\psi(u)=u^2/2$,
their Assumption~3(i) holds with $M_\psi=1$. We retain a generic failure
budget in the concentration argument underlying their Theorem~1. For
$\xi\in(0,1)$, write
\[
  L(\xi)=\log\!\left(\frac{C'_{\mathrm{u}}p}{\xi}\right).
\]
Under Condition~\ref{cond:design}, the finitely many sub-Gaussian score and
empirical-covariance concentration events used in that proof hold jointly
with probability at least $1-\xi$ after taking $C'_{\mathrm{u}}$ sufficiently
large and replacing $\log p$ by $L(\xi)$ in the penalties and sample-size
requirements. The remaining oracle inequalities in their proof are
deterministic on this joint event. Thus their argument, translated via Step~2
($n_0\to n_t$, $n_{\mathcal A_h}\to N_{t-1}$,
$n_{\mathcal A_h}+n_0\to N_t$, and their target sparsity to
$s=\norm{\beta_t^\star}_0$), gives the following confidence-level form: if
$h\sqrt{L(\xi)/n_t}=o(1)$ along the same asymptotic sequence,
$h\le C\sqrt s$,
$n_t\ge CL(\xi)$, and $N_{t-1}\ge CsL(\xi)$, then, with
\[
 \begin{aligned}
 \lambda_{w,t}&=C_w\sqrt{L(\xi)/N_t},
 & \lambda_{\Delta,t}&=C_\Delta\sqrt{L(\xi)/n_t},\\[-1mm]
 \norm{\widehat\beta_t-\beta_t^\star}_2
 &\lesssim \sqrt{\frac{sL(\xi)}{N_t}}
 +\left[\left(\frac{L(\xi)}{n_t}\right)^{1/4}h^{1/2}\wedge h\right],\\
 \norm{\widehat\beta_t-\beta_t^\star}_1
 &\lesssim s\sqrt{\frac{L(\xi)}{N_t}}+h .
 \end{aligned}
\]
with probability at least $1-\xi$. This is the confidence-level form of
their inequalities~(6)--(7) after the translation above; it retains the same
joint high-probability event. In particular, the argument does not, and cannot,
condition on $\widehat w_t$ and treat the current design as independent of it,
because $\widehat w_t$ is fit using the current batch through the pooled
objective.

\textbf{Step 5 (Squaring).} Squaring and using $(a+b)^2\le2a^2+2b^2$ to
absorb the cross term into the sum of squared orders gives
\[
 \norm{\widehat\beta_t-\beta_t^\star}_{\Sigma}^2
 \lesssim \frac{sL(\xi)}{N_t}+\left(h\sqrt{\frac{L(\xi)}{n_t}}\wedge h^2\right),
\]
on the same event, using the upper eigenvalue bound in
Condition~\ref{cond:design}. The second line in Step~4 gives the
$\ell_1$ bound.

\textbf{Parts (a) and (b).} For every $\xi\in(0,1)$, the standard Lasso
score and restricted-eigenvalue events under Condition~\ref{cond:design}
hold with probability at least $1-\xi$ when the penalties and sample sizes
use $L(\xi)$ \citep{bickel2009}. Thus part~(a) gives
$\norm{\widehat\beta_t^{\mathrm{cur}}-\beta_t^\star}_{\Sigma}^2
\lesssim sL(\xi)/n_t$ and
$\norm{\widehat\beta_t^{\mathrm{cur}}-\beta_t^\star}_1
\lesssim s\sqrt{L(\xi)/n_t}$ once $n_t\ge CsL(\xi)$. Its sample-size
requirement comes from the restricted-eigenvalue property over $s$-sparse
directions, unlike Tian--Feng's $n_0\ge C\log p$ condition, which controls
their debiasing step. Part~(b) additionally
assumes $\norm{w_t^\star}_0\le s_{w,t}$: since $\widehat w_t$ is exactly the
offline pooled Lasso targeting the $s_{w,t}$-sparse $w_t^\star$, the standard
Lasso bound gives $\norm{\widehat w_t-w_t^\star}_\Sigma^2
\lesssim s_{w,t}L(\xi)/N_t$ and
$\norm{\widehat w_t-w_t^\star}_1
\lesssim s_{w,t}\sqrt{L(\xi)/N_t}$ once
$N_t\ge Cs_{w,t}L(\xi)$. Adding
$D_t^2=\norm{w_t^\star-\beta_t^\star}_\Sigma^2$ by the triangle inequality
(again via $(a+b)^2\le2a^2+2b^2$) gives the current-task risk bound in
part~(b). This $\ell_1$ bound is used in
Lemma~\ref{lem:offset-stability} of Appendix~\ref{app:selection}.

\textbf{Step 6 (Anytime uniformization).} For each $t\ge1$, set
$\xi_t=\delta/[3t(t+1)]$ and choose $C_{\mathrm{u}}=3C'_{\mathrm{u}}$ in
Theorem~\ref{thm:transfer}, so that $L(\xi_t)=L_{t,\delta}$. The three
events establishing parts~(a), (b), and (c) at time $t$ then each fail with
probability at most $\xi_t$. A union bound over all times and all three
event classes gives
\[
  \Pbb\!\left(\bigcup_{t\ge1}\bigcup_{r\in\{\mathrm{cur},\mathrm{cum},\mathrm{ad}\}}E_{t,r}^c\right)
  \le 3\sum_{t\ge1}\xi_t
  =\delta,
\]
where $E_{t,r}$ denotes the event establishing assertion $r$ at time $t$.
Thus all displayed bounds hold simultaneously for every time at which their
respective assumptions are satisfied. No independence across time is required.

\section{Support-recovery consequences}
\label{app:selection}

This appendix records standard variable-selection consequences of the exact
reduction. It gives the additional conditions and proof for the signed-support
claim in Corollary~\ref{cor:shared-rate}, and records its related-task
analogue. These stronger conditions make clear that score
preservation transfers the relevant offline Lasso variable-selection guarantees.

\subsection{Conditions and proof for shared-coefficient signed-support recovery}
\label{app:shared-selection}

Let $S=\operatorname{supp}(\beta^\star)$ and
$\beta_{\min}=\min_{k\in S}|\beta_k^\star|$. In addition to the estimation
conditions of Corollary~\ref{cor:shared-rate}, suppose there are constants
$\kappa_S,M_S<\infty$ and $\gamma\in(0,1]$ such that
\[
  \lambda_{\min}(\Sigma_{SS})\ge\kappa_S,
  \qquad
  \norm{\Sigma_{SS}^{-1}}_\infty\le M_S,
  \qquad
  \norm{\Sigma_{S^cS}\Sigma_{SS}^{-1}}_\infty\le1-\gamma.
\]
Assume also that $N_t\gtrsim s^2\log p$ and
$\beta_{\min}\ge C_{\mathrm{sel}}\lambda_{w,t}$ for a sufficiently large
constant $C_{\mathrm{sel}}$. These are the additional conditions for its
signed-support conclusion. The constant implicit in
$\lambda_{w,t}\asymp\sqrt{\log p/N_t}$ is taken sufficiently large to
dominate the score noise.

\begin{proof}[Proof of the signed-support claim in Corollary~\ref{cor:shared-rate}]
Theorem~\ref{thm:exact} identifies $\widehat w_t$ with the offline pooled
Lasso. The stated irrepresentability, beta-min, and sample-size conditions
therefore yield the usual primal--dual-witness signed-support argument
\citep{wainwright2009}.
\end{proof}

\subsection{Related-task selection}

Write
\[
  S_{w,t}=\operatorname{supp}(w_t^\star),\qquad
  S_{\Delta,t}=\operatorname{supp}(\Delta_t^\star),\qquad
  S_{\beta,t}=\operatorname{supp}(\beta_t^\star),
\]
and let $s_{\Delta,t}=|S_{\Delta,t}|$. In addition to
Condition~\ref{cond:design}, $(\beta_1^\star,\ldots,\beta_t^\star)\in
\Theta_t(s,h)$, and the sparse-anchor condition
$\norm{w_t^\star}_0\le s_{w,t}$, suppose the following hold for target
time $t$:
\begin{enumerate}[label=(\roman*),leftmargin=*]
  \item For each $A\in\{S_{w,t},S_{\Delta,t}\}$,
  \[
    \lambda_{\min}(\Sigma_{AA})\ge\kappa_A,\qquad
    \norm{\Sigma_{AA}^{-1}}_\infty\le M_A,\qquad
    \norm{\Sigma_{A^cA}\Sigma_{AA}^{-1}}_\infty\le1-\gamma_A
  \]
  for positive constants $\kappa_A,M_A$ and $\gamma_A\in(0,1]$.
  \item The covariance entries are uniformly bounded,
  $\max_{k,\ell\le p}|\Sigma_{k\ell}|\le M_\Sigma$, and, under the anytime
  calibration of Theorem~\ref{thm:transfer}, the sample sizes and penalties
  obey
  \[
    N_t\gtrsim s_{w,t}^2L_{t,\delta},\qquad
    n_t\gtrsim s_{\Delta,t}^2L_{t,\delta},\qquad
    s_{w,t}\lambda_{w,t}\le c_\Delta\lambda_{\Delta,t},
  \]
  where $c_\Delta$ is sufficiently small relative to
  $\gamma_{S_{\Delta,t}}$.
  \item The sparse anchor and contrast are support-compatible,
  $S_{\beta,t}=S_{w,t}\cup S_{\Delta,t}$, and satisfy the beta-min conditions
  \[
    \min_{k\in S_{w,t}}|w_{t,k}^\star|\ge C_w\lambda_{w,t},\qquad
    \min_{k\in S_{\Delta,t}}|\Delta_{t,k}^\star|
    \ge C_\Delta\lambda_{\Delta,t},
  \]
  \[
    \min_{k\in S_{\beta,t}}|\beta_{t,k}^\star|
    \ge C_\beta(\lambda_{w,t}+\lambda_{\Delta,t})
  \]
  for sufficiently large constants determined by the selection geometry.
\end{enumerate}

\begin{lemma}[First-stage score stability]
\label{lem:offset-stability}
Under Condition~\ref{cond:design}, the sparse-anchor condition
$\norm{w_t^\star}_0\le s_{w,t}$ used in Theorem~\ref{thm:transfer}(b), and
the related-task selection conditions above, for a sufficiently large
penalty constant in
$\lambda_{\Delta,t}$,
\begin{equation}
  \left\|\frac{1}{n_t}X_t^{\mathsf T}X_t
  (\widehat w_t-w_t^\star)\right\|_\infty
  \le \frac{\gamma_{S_{\Delta,t}}}{8}\lambda_{\Delta,t}
  \label{eq:offset-stability}
\end{equation}
with probability $1-o(1)$.
\end{lemma}

\begin{proof}
By Theorem~\ref{thm:exact}, the first stage is the offline pooled Lasso
targeting the $s_{w,t}$-sparse $w_t^\star$, whose standard $\ell_1$ bound
$\norm{\widehat w_t-w_t^\star}_1=O_{\Pbb}(s_{w,t}\lambda_{w,t})$ is proved in
Appendix~\ref{app:transfer-proof} (the argument for Theorem~\ref{thm:transfer}(b)).
The sub-Gaussian design and bounded covariance entries give
$\max_{k,\ell\le p}|(X_t^{\mathsf T}X_t/n_t)_{k\ell}|=O_{\Pbb}(1)$.
The final condition in part~(ii), with a sufficiently large current-stage
penalty constant, therefore yields \eqref{eq:offset-stability}.
\end{proof}

\begin{corollary}[Current-task signed-support recovery]
\label{cor:transfer-selection}
Under Condition~\ref{cond:design}, $(\beta_1^\star,\ldots,\beta_t^\star)\in
\Theta_t(s,h)$, the sparse-anchor condition $\norm{w_t^\star}_0\le
s_{w,t}$, and the related-task selection conditions above,
\begin{equation}
  \Pbb\left\{
    \operatorname{supp}(\widehat\beta_t)=S_{\beta,t},
    \ \operatorname{sgn}(\widehat\beta_{t,S_{\beta,t}})=
    \operatorname{sgn}(\beta_{t,S_{\beta,t}}^\star)
  \right\}=1-o(1).
  \label{eq:transfer-selection}
\end{equation}
\end{corollary}

\begin{proof}
Theorem~\ref{thm:exact} turns the first stage into the offline pooled Lasso.
The selection geometry and beta-min condition give its primal--dual-witness
support result \citep{wainwright2009}. Lemma~\ref{lem:offset-stability} makes
the residual first-stage offset in \eqref{eq:target-correction} smaller than
the current-stage penalty, so the same argument applies to the contrast Lasso.
Support compatibility and
the final beta-min condition rule out cancellation when the two selected
components are added.
\end{proof}

\section{Proof of the metric decomposition}
\label{app:metric-proof}

\begin{proof}[Proof of Proposition~\ref{prop:metric-decomposition}]
For part~(a), subtract the two FWT definitions in Section~\ref{sec:stream} and use
the definitions of $\mathrm{BC}_k$ and $\mathrm{CG}_k$ in
Section~\ref{sec:metrics} at each task $k$. Part~(b) follows by reweighting the
historical and current MSE terms in those same definitions. For part~(c), substitute
the excess-risk identity in Section~\ref{sec:stream} into the definition of BWT
and apply the same
decomposition as \eqref{eq:conflict-decomposition}, restricted to tasks
$j<t$; the irreducible MSE terms cancel under equal batch weights.
\end{proof}

\section{Optional partial correction under a historical budget}
\label{app:partial-correction}

The two outputs define a simple continuum of deployment rules without changing
the transmitted state. For $\eta\in[0,1]$, let
\begin{equation}
 \widehat\beta_t(\eta)=\widehat w_t+\eta\widehat\Delta_t.
 \label{eq:partial-correction}
\end{equation}
Here $\eta=0$ returns the generalist and $\eta=1$ returns the fully adapted
specialist. The stored state evaluates the historical squared loss, up to a
constant independent of $b$, through
\begin{equation}
 \widehat L_{<t}^{\mathrm{hist}}(b)=
 \frac{1}{2N_{t-1}}\left\{
 b^{\mathsf T}H_{t-1}b-2b^{\mathsf T}H_{t-1}u_{t-1}\right\}.
 \label{eq:historical-state-loss}
\end{equation}
One optional budgeted rule chooses the current-loss-optimal value over a fixed
grid subject to
\begin{equation}
 \widehat L_{<t}^{\mathrm{hist}}\{\widehat\beta_t(\eta)\}-
 \widehat L_{<t}^{\mathrm{hist}}(\widehat w_t)\leq \varepsilon_t.
 \label{eq:historical-budget}
\end{equation}
This is a deployment choice, not a new state-update algorithm.

Its population trade-off is explicit under the common-covariance model. Let
$\rho_t=n_t/N_t$, write $d_t^\star=\beta_t^\star-w_{t-1}^\star$, and set
$b_t^\eta=w_t^\star+\eta(\beta_t^\star-w_t^\star)$. Replacing the estimates in
\eqref{eq:partial-correction} by their population targets yields
\begin{equation}
 \begin{aligned}
 \risk_t(b_t^\eta)
 &=\{1-\gamma_t(\eta)\}^2\norm{d_t^\star}_{\Sigma}^2,\\
 \overline\mse_{<t}(b_t^\eta)-\overline\mse_{<t}(w_{t-1}^\star)
 &=\gamma_t(\eta)^2\norm{d_t^\star}_{\Sigma}^2,
 \end{aligned}
 \label{eq:partial-frontier}
\end{equation}
where $\gamma_t(\eta)=\rho_t+(1-\rho_t)\eta$. Equation
\eqref{eq:partial-frontier} separates the unavoidable conflict of a
single head from the controllable choice of how far to adapt toward the new
task.

\section{Real-Data Application: Climate-Stream Prediction}
\label{app:realdata}

This appendix reports an independent real-data confirmation of
Theorem~\ref{thm:exact} and Proposition~\ref{thm:direct-accumulation} on a
chronological climate-prediction stream, using a separate implementation of
the same estimators written for this application. Because the true
coefficient is unknown here, the analysis emphasizes the score-transmission
diagnostic and final variable identification rather than estimation error
against a ground truth.

\subsection{Data and stream construction}

The predictor is the NCEP--NCAR Reanalysis 1 sea-level-pressure field
\citep{kalnay1996ncep} over the North Pacific/North America window
$0^\circ$--$70^\circ$N, $150^\circ$E--$120^\circ$W, on a $2.5^\circ$ grid
flattened to $p=1073$ grid cells. The response is the one-week-ahead anomaly
in weekly mean temperature at Los Angeles International Airport (NOAA
station USW00023174, $33.938^\circ$N, $118.387^\circ$W)
\citep{menne2012ghcn}; missing daily TAVG values are filled by
$(\mathrm{TMAX}+\mathrm{TMIN})/2$ before weekly averaging. We retain the
extended winter months November--March and pair each week's pressure field
with the following week's temperature anomaly within the same winter, so no
pair spans two winters.

To avoid leakage, all preprocessing is fixed on a 1980--1984 reference
period (week-of-year SLP and temperature climatology, per-cell SLP scaling,
with reference standard deviations below $10^{-8}$ floored at 1) before the
retrospective stream begins; these reference winters are never used for
fitting or testing. Starting in winter 1985, each of $T=8$ chronological
tasks pools four training winters ($n_t\approx81$--$83$ weekly pairs) and
holds out a fifth winter as a task-specific test block ($20$--$21$ pairs);
the test winters are 1989, 1994, 1999, 2004, 2009, 2014, 2019, and 2024. A
final future block (winters 2025--2026, 41 pairs) is reserved for
out-of-sample prediction and is disjoint from every historical test winter.
Because training-block sizes are nearly but not exactly equal, the
transmitted state uses a task-balanced cumulative objective: each task
contributes its empirical average loss.

\subsection{Evaluation and tuning}

We compare Current-only Lasso, Offline pooled Lasso, Direct center, and
SP-CL under a common penalty schedule $\lambda_t=\lambda_1/\sqrt t$, with
$\lambda_1=0.20$ selected by rolling-origin validation on the training
blocks only (grid $\{0.05,0.10,0.20,0.30,0.50,0.80\}$; all held-out test
winters and the future block are excluded from tuning), together with a
zero-anomaly and a persistence (current-week temperature) baseline. For
task-specific test block $j$, let $A_{t,j}$ denote the test mean squared
forecast error (MSFE) of the state deployed after task $t$, evaluated on
block $j\le t$; define the historical forgetting
$\mathrm{MaxF}_t=\max_{j<t}[A_{t,j}-A_{j,j}]_+$ and the historical regret
relative to Offline pooled Lasso
$\mathrm{MReg}_t=\max_{j<t}[A_{t,j}-A_{t,j}^{\mathrm{off}}]_+$, summarized by
$\mathrm{MReg}_T^{\mathrm{hist}}=\max_{2\le t\le T}\mathrm{MReg}_t$ (and
similarly $\mathrm{MaxF}_T^{\mathrm{hist}}$). Historical-test uncertainty
uses a paired percentile bootstrap (2000 replications) over the eight test
winters; the short future block instead uses a winter-stratified
moving-block bootstrap (four-week blocks).

\subsection{Prediction and historical risk}

\begin{table}[H]
\centering
\caption{Climate-stream results with historical test blocks. ``Future
MSFE'' uses the disjoint future block (winters 2025--2026). ``Hist.\ MSFE''
is the final-state average over the eight held-out historical test winters.
$\mathrm{MaxF}_T^{\mathrm{hist}}$ and $\mathrm{MReg}_T^{\mathrm{hist}}$ are
the maxima over time of historical forgetting and of historical regret
relative to Offline pooled Lasso, respectively.}
\label{tab:realdata-climate}
\scriptsize
\setlength{\tabcolsep}{4pt}
\renewcommand{\arraystretch}{1.1}
\begin{tabular}{@{}lcccccc@{}}
\toprule
Method & Future MSFE & $R^2_{\mathrm{OOS}}$ & Hist.\ MSFE &
$\mathrm{MaxF}_T^{\mathrm{hist}}$ & $\mathrm{MReg}_T^{\mathrm{hist}}$ & Final support \\
\midrule
Zero anomaly & 5.256 & 0.000 & -- & -- & -- & -- \\
Persistence & 5.902 & $-0.123$ & -- & -- & -- & -- \\
Current-only Lasso & 4.944 & 0.059 & 3.709 & 2.475 & 1.911 & 13 \\
Offline pooled Lasso & 3.914 & 0.255 & 3.551 & 0.891 & 0.000 & 20 \\
Direct center & 4.524 & 0.139 & 3.980 & 1.639 & 1.353 & 10 \\
SP-CL & 3.914 & 0.255 & 3.551 & 0.891 & 0.000 & 20 \\
\bottomrule
\end{tabular}
\end{table}

Table~\ref{tab:realdata-climate} shows the sparse SLP models improving on
the zero-anomaly and persistence baselines on the future block. SP-CL
coincides with Offline pooled Lasso to numerical tolerance, as
Theorem~\ref{thm:exact} predicts: the empirical score-transmission gap is
$3.6\times10^{-11}$ for SP-CL against $0.391$ for Direct center at the final
task, mirroring the mechanism in Figure~\ref{fig:mechanism-noisy}. The future
block contains only 41 weekly pairs from two winters, so these comparisons
are descriptive: winter-specific future MSFEs are 1.555 (2025) and 6.391
(2026) for Offline pooled Lasso/SP-CL against 1.480 and 7.720 for Direct
center, and a moving-block bootstrap gives a mean paired MSFE difference
(Direct center minus Offline pooled Lasso) of 0.610 on the combined future
block, 95\% interval $[-0.164,1.103]$; Offline pooled Lasso against
zero-anomaly and persistence gives mean differences of $-1.342$
($[-2.106,0.039]$) and $-1.988$ ($[-5.046,-0.628]$), respectively.

The historical test blocks give a more direct forgetting check. A paired
percentile bootstrap over the eight held-out winters gives an average final
historical-test excess MSFE for Direct center relative to Offline pooled
Lasso of $0.430$, 95\% interval $[0.059,0.861]$; the corresponding
Current-only excess is $0.158$, interval $[-0.226,0.553]$ (includes zero).
Direct center selects fewer grid cells (10, against 20 for Offline pooled
Lasso/SP-CL) -- consistent with the accumulated-shrinkage mechanism of
Proposition~\ref{thm:direct-accumulation} rather than superior variable
identification: the Offline pooled Lasso/SP-CL supports coincide exactly,
while Direct center's support has no exact overlap with them, even though
70\% of its selected cells lie within $5^\circ$ and 80\% within $10^\circ$
of a shared-support cell, and the two fitted-value series remain highly
correlated ($0.952$ historical, $0.967$ future).

\subsection{Spatial diagnostics and limitations}

Figure~\ref{fig:realdata-dashboard} plots the future-block MSFE, forgetting
and regret curves, and a batchwise covariance-distance diagnostic
$\lVert\widehat\Sigma_j-\widehat\Sigma_k\rVert_F/(\lVert\widehat\Sigma_j\rVert_F+\lVert\widehat\Sigma_k\rVert_F)$;
its between-task mean ($0.268$) and within-task split-half mean ($0.295$)
are close, and a sample-size-matched calibration in the leading ten
empirical PC directions ($0.306$ versus $0.313$) remains close, so this
diagnostic is descriptive rather than confirmatory evidence of covariate
shift. Figure~\ref{fig:realdata-heatmap} compares final supports spatially,
and Figure~\ref{fig:realdata-signed} additionally shows coefficient signs.
These are exploratory summaries of correlated pressure-grid proxies, not a
causal attribution, and this application illustrates score transmission on
one real chronological stream; it neither verifies the common-covariance
approximation of Section~\ref{sec:oracle-regime} nor rules out
climate-regime drift more generally.

\begin{figure}[H]
\centering
\includegraphics[width=0.98\textwidth]{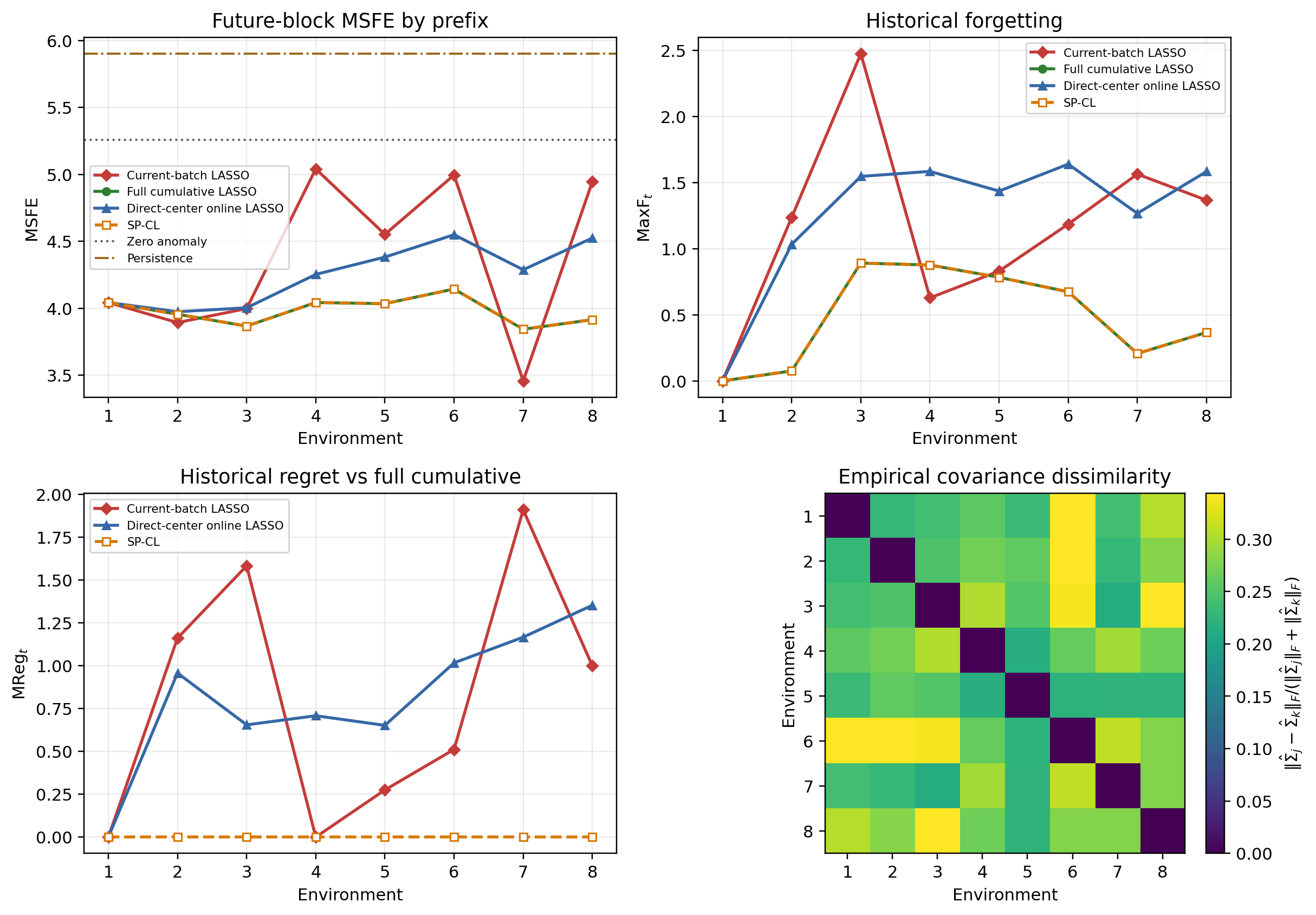}
\caption{Climate-stream historical test-block analysis. Top left: raw
future-block MSFE for each method, with zero-anomaly and persistence
baselines. Top right: historical forgetting $\mathrm{MaxF}_t$. Bottom left:
historical regret $\mathrm{MReg}_t$ relative to Offline pooled Lasso.
Bottom right: pairwise empirical covariance dissimilarities between
chronological training tasks. The common penalty is $\lambda_t=0.20/\sqrt
t$.}
\label{fig:realdata-dashboard}
\end{figure}

\begin{figure}[H]
\centering
\includegraphics[width=0.98\textwidth]{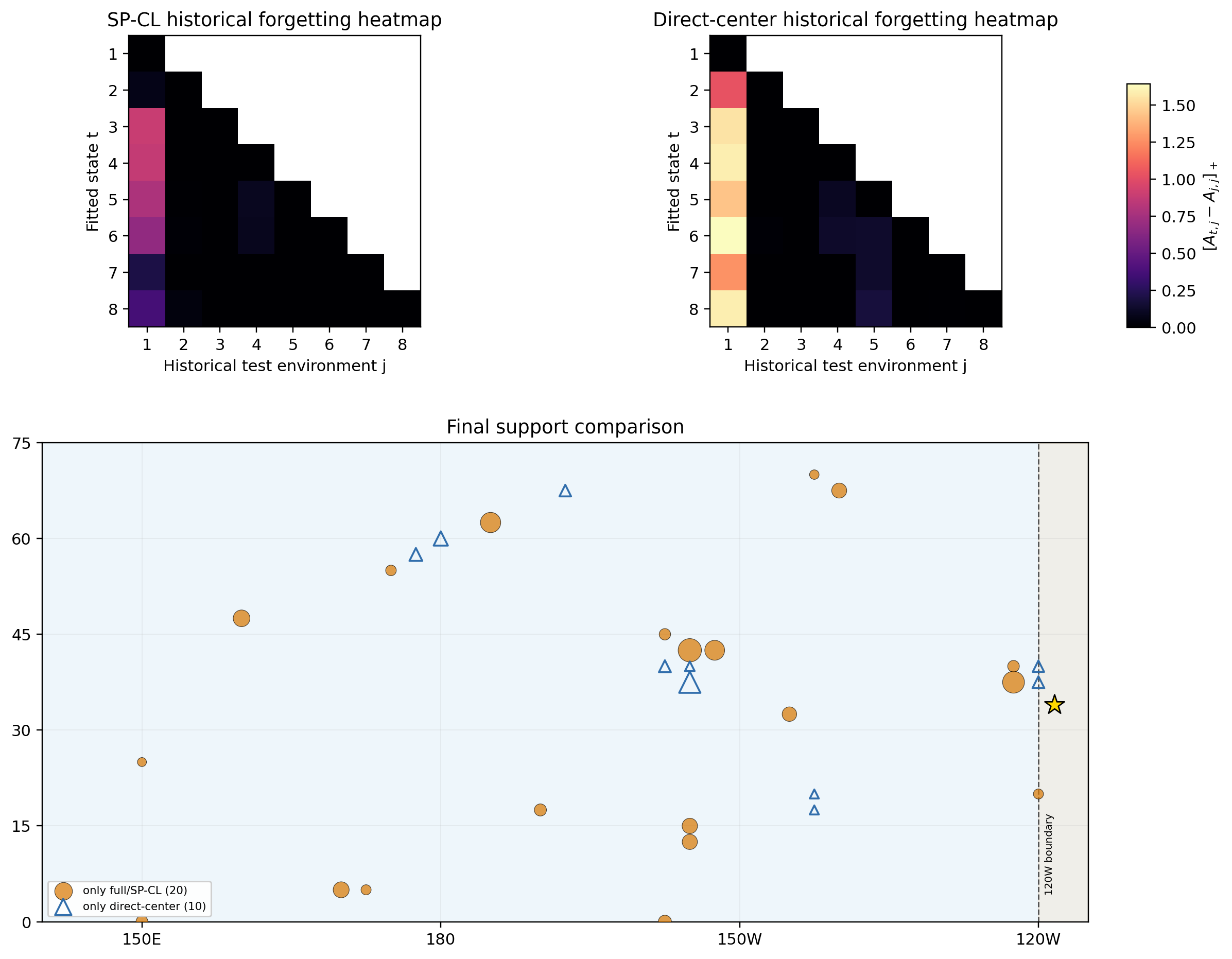}
\caption{Historical forgetting heatmaps and final support comparison. The
heatmaps show $[A_{t,j}-A_{j,j}]_+$ for SP-CL and Direct center on a common
color scale. The map compares final supports: orange points are selected
only by Offline pooled Lasso/SP-CL, blue triangles only by Direct center;
point size uses a common coefficient-magnitude scale, the dashed line marks
the eastern edge of the SLP predictor window at $120^\circ$W, and the star
marks Los Angeles immediately east of that window.}
\label{fig:realdata-heatmap}
\end{figure}

\begin{figure}[H]
\centering
\includegraphics[width=0.98\textwidth]{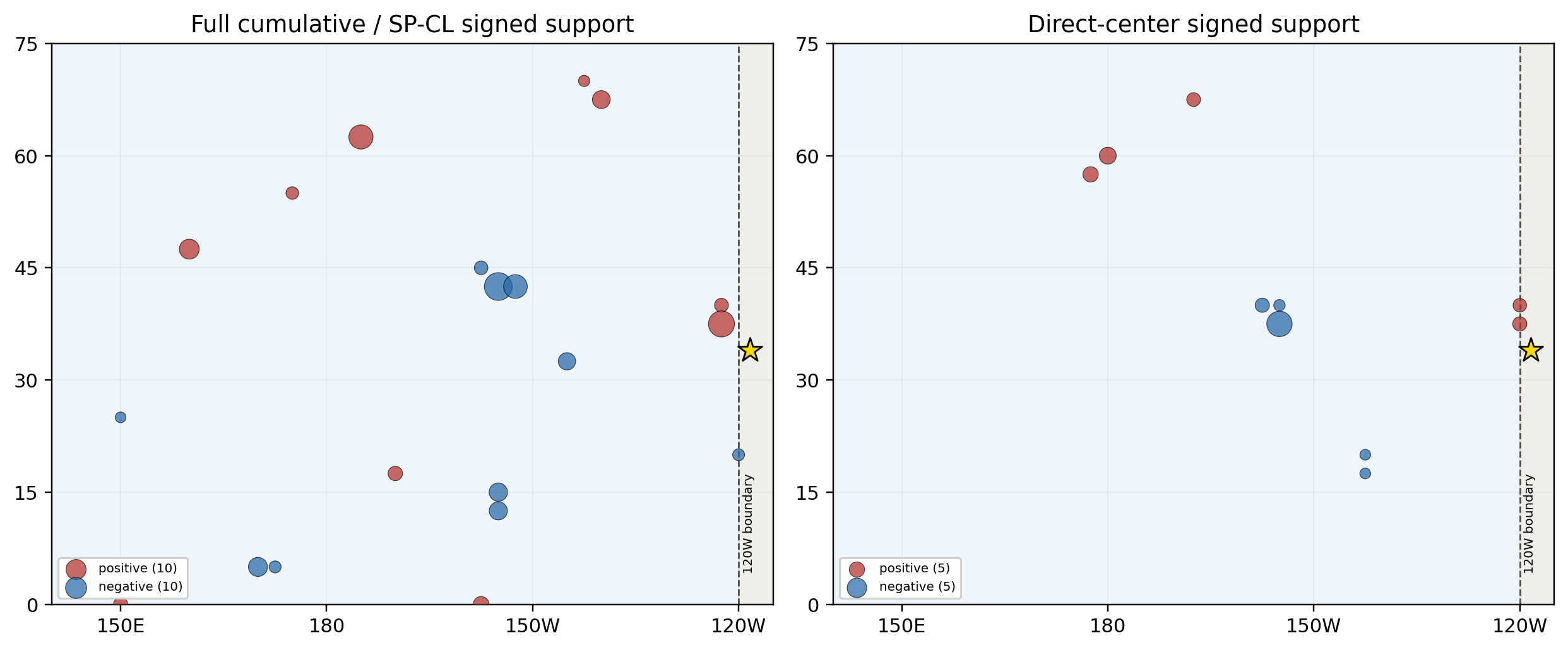}
\caption{Signed final support maps for the climate-stream application. Red
points denote positive coefficients, blue points denote negative
coefficients, with point size on a common absolute-magnitude scale across
panels. These maps separate coefficient sign from the method-overlap
comparison in Figure~\ref{fig:realdata-heatmap}.}
\label{fig:realdata-signed}
\end{figure}

\end{document}